%% file: main.tex
\documentclass[11pt]{article} 
\usepackage{hyperref}
\usepackage{url}
\usepackage{smile}
\usepackage{graphicx} 
\usepackage{algpseudocode}
\usepackage{algorithm}
\usepackage{todonotes}
\usepackage{epstopdf}
\usepackage[margin=1in]{geometry}
\usepackage[normalem]{ulem}
\usepackage[export]{adjustbox}
\usepackage{mathtools, cuted}
\usepackage{natbib}
\usepackage{bbm}
\usepackage{wrapfig}
\usepackage{subcaption}
\usepackage{caption}
\usepackage{enumitem}
\usepackage[parfill]{parskip}
\usepackage{authblk}
\usepackage{amsmath}

\hypersetup{
    colorlinks=true,
    linkcolor=blue,
    anchorcolor=blue, 
    citecolor=blue,
}

\usepackage{kpfonts}
\DeclareMathAlphabet{\mathsf}{OT1}{cmss}{m}{n}
\SetMathAlphabet{\mathsf}{bold}{OT1}{cmss}{bx}{n}

\providecommand{\norm}[1]{\|#1\|}

\begin{document}

\title{\huge \bf{Robust Multi-Agent Reinforcement Learning via Adversarial Regularization:
Theoretical Foundation and Stable Algorithms}}

\author[1]{Alexander Bukharin}
\author[1]{Yan Li}
\author[1]{Yue Yu}
\author[1]{Qingru Zhang}
\author[2]{Zhehui Chen}
\author[3]{Simiao Zuo}
\author[1]{Chao Zhang}
\author[4]{Songan Zhang}
\author[1]{Tuo Zhao}
\affil[1]{Georgia Institute of Technology}
\affil[2]{Google}
\affil[3]{Microsoft}
\affil[4]{Ford Motor Company}


\maketitle

\begin{abstract}
\noindent Multi-Agent Reinforcement Learning (MARL) has shown promising results across several domains. Despite this promise, MARL policies often lack robustness and are therefore sensitive to small changes in their environment. This presents a serious concern for the real world deployment of MARL algorithms, where the testing environment may slightly differ from the training environment. In this work we show that we can gain robustness by controlling a policy's Lipschitz constant, and under mild conditions, establish the existence of a Lipschitz and close-to-optimal policy. 
Based on these insights, we propose a new robust MARL framework, ERNIE, that promotes the Lipschitz continuity of the policies with respect to the state observations and actions by adversarial regularization. The ERNIE framework provides robustness against noisy observations, changing transition dynamics, and malicious actions of agents. However, ERNIE's adversarial regularization may introduce some training instability. To reduce this instability, we reformulate adversarial regularization as a Stackelberg game. We demonstrate the effectiveness of the proposed framework with extensive experiments in traffic light control and particle environments. In addition, we extend ERNIE to mean-field MARL with a formulation based on distributionally robust optimization that outperforms its non-robust counterpart and is of independent interest.
Our code is available at \url{https://github.com/abukharin3/ERNIE}.
\end{abstract}

\section{Introduction}
\label{introduction}

\input{source_files/intro.tex}

\section{Background}
\label{sec:background}


\input{source_files/background.tex}

\section{From Lipschitz Continuity to Robustness}
\label{sec:theory}
\vspace{-0.1cm}
\input{source_files/theory.tex}

\section{Method}
\label{sec:method}
 
\input{source_files/method.tex}



\section{Experiments}
\label{sec:experiments}
\input{source_files/experiments.tex}

\section{Discussion}
ERNIE is motivated by smoothness, but real-world environments are not always smooth. In section \ref{sec:theory}, we hypothesize that most environments are at least partially smooth, implying that smoothness can serve as useful prior knowledge while providing robustness (our experiments validate this). To increase ERNIE's flexibility, future work could adaptively select $\lambda$ based on the current state to allow for state-dependent smoothness.

\section{Acknowledgments}
We would like to thank Haoming Jiang and Ethan Wang for their early contributions to this project.

\bibliographystyle{unsrtnat}
\bibliography{refs}

\newpage
\appendix
\onecolumn

\section{Proofs}




\begin{proof}[Proof of Theorem \ref{theorem_smooth_vals}]
From the definition of Q-function, we have 
\begin{align*}
  \abs{ Q^{\pi}(s,a) - Q^{\pi}(s',a)}
    & \leq \abs{ r(s,a) - r(s', a)} + \gamma \abs{ \tsum_{s_1 \in \cS} \PP(s_1|s,a) V^{\pi}(s_1) 
    - \tsum_{s_1 \in \cS} \PP(s_1|s',a) V^{\pi}(s_1)
    } \\
    & \leq L_r \norm{s - s'} 
    + \tfrac{\gamma  L_{\PP}}{1-\gamma}  \norm{s - s'},
\end{align*}
where the last inequality also uses the fact that $\norm{V^{\pi}}_{\infty} \leq \tfrac{1}{1-\gamma}$.


From the relation of $Q^{\pi}$ and $V^{\pi}$, we have 
\begin{align*}
    & |V^{\pi}(s) - V^{\pi}(s')| \\
    = & 
    \abs{
    \inner{Q^{\pi}(s, \cdot)}{\pi(\cdot|s)}
    -
    \inner{Q^{\pi}(s', \cdot)}{\pi(\cdot|s')}
    } \\
    = & 
    \abs{\inner{Q^{\pi}(s,\cdot)}{\pi(\cdot|s) - \pi(\cdot|s')}
    + \inner{Q^{\pi}(s,\cdot) - Q^{\pi}(s',\cdot)}{\pi(\cdot|s')
    } }
    \\
  \leq & 
  \abs{\inner{Q^{\pi}(s,\cdot)}{\pi(\cdot|s) - \pi(\cdot|s')}}
  + 
  \abs{\inner{Q^{\pi}(s,\cdot) - Q^{\pi}(s',\cdot)}{\pi(\cdot|s')
    }}
  \\
  \leq & 
  \tfrac{L_{\pi} }{1-\gamma} \norm{s - s'} 
  + L_Q \norm{s - s'}.
\end{align*}
\end{proof}

\begin{proof}[Proof of Theorem \ref{thrm_existence_smooth_policy}]
Let $Q^* \in \RR^{\abs{\cS} \times \abs{\cA}} $ denotes the optimal state-action value function. 
Let $\pi^*$ denote any optimal policy.
From the Bellman optimality condition, it is clear that 
\begin{align}
\label{bellman_opt_condition}
    \mathrm{sup}(\pi(\cdot|s)) \subseteq 
    \mathrm{Argmax}_{a \in \cA} Q^*(s,a).
\end{align}
From the performance difference lemma, we obtain that for any policy $\pi$, the optimality gap of $\pi$ can be bounded by 
\begin{align*}
  0 \leq  V^{\pi^*}(s) - V^{\pi}(s) 
 &= - \rbr{ V^{\pi}(s) - V^{\pi^*}(s) }\\
 & = - \EE_{s' \sim d_s^\pi} \sbr{
 \inner{Q^{\pi^*}(s', \cdot)}{\pi(\cdot|s') - \pi^*(\cdot|s)}
 } \\
 & = \EE_{s' \sim d_s^\pi} \sbr{
 \inner{Q^{\pi^*}(s', \cdot)}{\pi^*(\cdot|s') - \pi(\cdot|s')}
 } \\
 & \leq 
 \sup_{s \in \cS} \inner{Q^{\pi^*}(s, \cdot)}{\pi^*(\cdot|s) - \pi(\cdot|s)}
\end{align*}
where the last inquality uses \eqref{bellman_opt_condition}, which implies that inner product is non-negative for every $s \in \cS$.

Now consider the special policy 
\begin{align*}
    \pi_\eta(\cdot|s) = \mathrm{Softmax} (\eta Q^*(s, \cdot)),
\end{align*}
where operator $\mathrm{Softmax}: \RR^{\abs{\cA}} \to \RR^{\abs{\cA}} $  is defined as 
$
    \sbr{\mathrm{Softmax}(x)}_i = \exp(x_i) / \tsum_{j} \exp(x_j).
$

For any $\epsilon > 0$, let us define 
\begin{align*}
    \cA_\epsilon = \cbr{
    a \in \cA: ~ Q^*(s,a) \leq \max_{a' \in \cA} Q^*(s,a') - \epsilon
    }.
\end{align*}
Consequently, we have 
\begin{align*}
    \inner{Q^{\pi^*}(s, \cdot)}{\pi^*(\cdot|s) - \pi_\eta(\cdot|s)}
    \leq \tfrac{\epsilon}{1-\gamma} + \tsum_{a \in \cA_\epsilon} \pi_\eta(a|s).
\end{align*}
It then suffices to set $\eta$ properly to control the second term above.
Specifically, we have 
\begin{align*}
    \tsum_{a \in \cA_\epsilon} \pi_\eta(a|s)
    \leq 
    \tfrac{\tsum_{a \in \cA_\epsilon} \exp(-\eta \epsilon)}
    {1 + \abs{\cA^c_\epsilon} \exp(-\eta \epsilon) }
    \leq \abs{\cA^c_\epsilon} \exp(-\eta \epsilon),
\end{align*}
By setting $\eta = \log \abs{\cA} / \epsilon $,
we immediately obtain that $\tsum_{a \in \cA_\epsilon} \pi_\eta(a|s) \leq \epsilon$, and hence 
\begin{align*}
     \inner{Q^{\pi^*}(s, \cdot)}{\pi^*(\cdot|s) - \pi_\eta(\cdot|s)}
    \leq \tfrac{\epsilon}{1-\gamma} + \tsum_{a \in \cA_\epsilon} \pi_\eta(a|s)
    \leq \tfrac{2\epsilon}{1-\gamma}, ~ \forall s \in \cS.
\end{align*}
Hence, we obtain that 
$
     V^{\pi^*}(s) - V^{\pi}(s)  \leq \tfrac{2\epsilon}{1-\gamma}, ~ \forall s \in \cS.
$

It remains to show that $\pi_\eta(\cdot|s)$ with the $\eta = \log \abs{\cA} /\epsilon$ is indeed Lipschitz continuous with respect to state $s$.
To this end, 
let us denote the Jacobian matrix of $\mathrm{Softmax}$ at point $x$ as $\cJ_x$.
Simple calculation  yields that 
\begin{align*}
    [\cJ_x]_{i,i} &= \tfrac{\exp(x_i) \tsum_{j \neq i} \exp(x_j)}{(\tsum_j \exp(x_j))^2},  \\
    [\cJ_x]_{i,j} & = 
    - \tfrac{\exp(x_i) \exp(x_j)}{(\tsum_j \exp(x_j))^2}.
\end{align*}
From the intermediate value theorem, we obtain that 
$
    \norm{\mathrm{Softmax}(x) - \mathrm{Softmax}(y) }_1
    \leq \norm{\cJ_z}_1 \norm{x - y}_1,
$
for $z  = \alpha x + (1-\alpha) y$ and $\alpha \in [0,1]$.
Since it is clear that 
$
\norm{\cJ_z}_1 \leq 1
$
, we conclude that 
\begin{align*}
    \norm{\pi_\eta(\cdot|s) - \pi_\eta(\cdot|s')}_1
    \leq \eta  \norm{ Q^*(s,\cdot) - Q^*(s', \cdot)}_1
    \overset{(a)}{\leq} \eta  L_Q \abs{\cA}  \norm{s - s'}
    \leq  \abs{\cA}  \log \abs{\cA}  L_Q \norm{s - s'} / \epsilon,
\end{align*}
where inequality $(a)$ applies Theorem \ref{theorem_smooth_vals}.
\end{proof}

\begin{proof}[Proof of Theorem \ref{theorem_v_p}]
We first recall that for any stationary policy $\pi$, the value function $V^\pi$ at any state $\overline{s}$ admits the following description, 
\begin{align}
    V^{\pi}(\overline{s})  & = \EE \sbr{\tsum_{t=0}^\infty \gamma^t r(s_t, a_t) | s_0 = \overline{s}} \nonumber \\
    & =  \tsum_{s\in \cS, a\in \cA} \tsum_{t=0}^\infty \gamma^t \PP^{\pi} (s_t =s, a_t = a | s_0 = \overline{s}) r(s,a) \nonumber \\
    & =  \tsum_{s\in \cS, a\in \cA} \tsum_{t=0}^\infty
  \gamma^t \PP^\pi(s_t = s| s_0 = \overline{s}) \pi(a|s) r(s,a)
    . \label{value_standard_decomp}
\end{align}

Similarly, given the definition of $V^{\tilde{\pi}}$ for the non-stationary policy, we know that 
\begin{align}
    V^{\tilde{\pi}}(\overline{s}) 
    = \tsum_{s \in \cS, a \in \cA}
    \tsum_{t=0}^\infty \gamma^t  \PP^{\tilde{\pi}}(s_t = s|s_0 = \overline{s})  \tilde{\pi}(a | s) r(s,a). \label{value_perturb_decomp}
\end{align}

By definition, we have the following observations,
\begin{align*}
   \PP^{\tilde{\pi}, t} \coloneqq   \PP^{\tilde{\pi}}(s_t = \cdot|s_0 = \overline{s}) 
    = \prod_{i=0}^{t-1} \PP^{\tilde{\pi}}_i e(\overline{s}),
    ~~
  \PP^{\pi, t} \coloneqq  \PP^{\pi} (s_t = \cdot|s_0 = \overline{s}) = 
    \rbr{\PP^\pi}^t e(\overline{s}),
\end{align*}
where $\PP_i^{\tilde{\pi}}(s', s) \coloneqq \tsum_{a \in \cA} \tilde{\pi}_i(a|s) \PP(s'|s,a)$,
and $\PP_i^{\pi}(s', s) \coloneqq \tsum_{a \in \cA} \pi_i(a|s) \PP(s'|s,a)$,
and $e(s) \in \RR^{\abs{\cS}}$ denotes the one-hot vector with non-zero entry corresponding to the state $s$
.
Hence 
\begin{align}
    & \norm{ \PP^{\tilde{\pi}, t} -  \PP^{\pi, t}}_1 \nonumber \\ 
      = &  \norm{
    \prod_{i=0}^{t-1} \PP^{\tilde{\pi}}_i e(\overline{s}) - \rbr{\PP^\pi}^t e(\overline{s})
    }_1 \nonumber \\
     \leq &
     \norm{
    \prod_{i=0}^{t-1} \PP^{\tilde{\pi}}_i - \rbr{\PP^\pi}^t 
    }_1  \nonumber\\
     \leq &
    \norm{
    \prod_{i=0}^{t-1}  \PP^{\tilde{\pi}}_i
    - \PP^\pi \prod_{i=1}^{t-1}  \PP^{\tilde{\pi}}_i
    + \PP^\pi \prod_{i=1}^{t-1}  \PP^{\tilde{\pi}}_i
    - \rbr{\PP^\pi}^2 \prod_{i=2}^{t-1}  \PP^{\tilde{\pi}}_i + \cdots 
    + \rbr{\PP^\pi}^{t-1} \PP^{\tilde{\pi}}_{t-1}
    - \rbr{\PP^\pi}^{t}
    }_1  \nonumber\\
    \leq & 
    \norm{\prod_{i=0}^{t-1}  \PP^{\tilde{\pi}}_i -
     \PP^\pi \prod_{i=1}^{t-1}  \PP^{\tilde{\pi}}_i
    }_1 
    + \norm{\PP^\pi \prod_{i=1}^{t-1}  \PP^{\tilde{\pi}}_i
    - \rbr{\PP^\pi}^2 \prod_{i=2}^{t-1}  \PP^{\tilde{\pi}}_i}_1
    + \cdots + 
    \norm{\rbr{\PP^\pi}^{t-1} \PP^{\tilde{\pi}}_{t-1}
    - \rbr{\PP^\pi}^{t}}_1 \label{ineq_each_time_total_transit_diff}
\end{align}
To handle each term above, we make use of the following lemma. 
\begin{lemma}\label{lemma_matrix_norm_bound}
For any $P, Q \in \RR^d$ that are left stochastic matrices, and any matrix $\Delta$ of the same dimension, we have 
\begin{align*}
\norm{P \Delta Q}_1 
\leq  \norm{\Delta}_1.
\end{align*}
\end{lemma}

\begin{proof}
Note that $\norm{\cdot}_1$ is an induced norm and hence is sub-multiplicative. 
In addition, we have $\norm{P}_1 = \norm{Q}_1 = 1$ since they are left stochastic matrices. 
We have 
\begin{align*}
\norm{P \Delta Q}_1
 \leq
\norm{P \Delta}_1 
\leq \norm{\Delta}_1.
\end{align*}
\end{proof}

Now for the $k$-th term in inequality \eqref{ineq_each_time_total_transit_diff}, it can be rewritten and bounded as 
\begin{align*}
\norm{
\rbr{\PP^\pi }^{k-1} \rbr{
\PP^{\tilde{\pi}}_k - \PP^\pi 
}
\prod_{i={k+1}}^{t-1} \PP_i^{\tilde{\pi}}
}_1
\overset{(a)}{\leq} 
\norm{\PP^{\tilde{\pi}}_k - \PP^\pi }_1
\overset{(b)}{\leq} L_\pi \epsilon.
\end{align*}
where the inequality $(a)$ follows from Lemma \ref{lemma_matrix_norm_bound};
and $(b)$ use the following fact
\begin{align*}
   \tsum_{s' \in \cS} \abs{\PP_k^{\tilde{\pi}}(s', s) - \PP^\pi(s',s)  }
    & =\tsum_{s' \in \cS} \abs{
    \tsum_{a \in \cA} \rbr{ \tilde{\pi}_k(a|s) - \pi(a|s)} \PP(s'|s,a) 
    } \\
    & \leq 
    \tsum_{a \in \cA} \abs{ \tilde{\pi}_k(a|s) - \pi(a|s)}  \tsum_{s' \in \cS}
 \PP(s'|s,a) \\
 & = \norm{\tilde{\pi}_k(\cdot |s) - \pi(\cdot|s)}_1
 \leq L_\pi \epsilon,
\end{align*}
together with the definition of matrix $\norm{\cdot}_1$-norm.
Thus we obtain 
\begin{align}\label{each_step_transit_diff_simplified}
\norm{ \PP^{\tilde{\pi}, t} -  \PP^{\pi, t}}_1
\leq t L_\pi \epsilon.
\end{align}

Given \eqref{each_step_transit_diff_simplified}, we can further obtain that
\begin{align}
   &  \abs{
    \tsum_{t=0}^\infty \gamma^t  \tsum_{s\in \cS}
    \PP^\pi(s_t = s| s_0 = \overline{s}) \tsum_{a\in \cA} \pi(a|s) r(s,a)
    - 
    \tsum_{t=0}^\infty \gamma^t \tsum_{s \in \cS}
    \PP^{\tilde{\pi}}(s_t = s| s_0 = \overline{s}) \tsum_{a \in \cA} \pi(a|s) r(s,a)
    }  \nonumber\\
    \leq &
   \tsum_{t=0}^\infty \gamma^t
   \norm{ \PP^{\pi, t}
   -
    \PP^{\tilde{\pi}, t} }_1 \nonumber \\
    \leq & 
    \tsum_{t= 0}^\infty \gamma^t \cdot t L_\pi \epsilon 
    \leq \frac{L_\pi \epsilon}{(1-\gamma)^2}. \label{value_diff_decomp_1}
\end{align}

In addition, it is also clear that 
\begin{align}
\abs{
    \tsum_{t=0}^\infty \gamma^t \tsum_{s \in \cS} \tsum_{a \in \cA}
    \PP^{\tilde{\pi}}(s_t = s| s_0 = \overline{s})  (\pi(a|s) -\tilde{\pi}_t(a|s)) r(s,a)
}
\leq \frac{L_\pi \epsilon}{1-\gamma}. \label{value_diff_decomp_2}
\end{align}
Hence by combining \eqref{value_standard_decomp}, \eqref{value_perturb_decomp}, \eqref{value_diff_decomp_1} and \eqref{value_diff_decomp_2}, we conclude that 
\begin{align*}
   \abs{V^{\pi}(\overline{s}) - V^{\tilde{\pi}}(\overline{s})} 
     \leq \frac{2 L_\pi \epsilon}{(1-\gamma)^2}.
\end{align*}





From the relation between $Q^\pi$ and $V^\pi$, and the above inequality, we have
\begin{align*}
\abs{Q^{\pi}(s,a) - Q^{\tilde{\pi}}(s,a)} &= 
\gamma \tsum_{s' \in \cS} \PP (s' | s,a) \abs{V^{\pi} (s') - V^{\tilde{\pi}}(s')}\leq \tfrac{  2L_\pi \epsilon }{(1-\gamma)^2}.
\end{align*}
\end{proof}

\begin{theorem}[Function approximation with Lipschitz continuity]\label{approx-thm}
Suppose that the target function $f^*$ satisfies
\begin{align*}
f^*\in W^{\alpha,\infty}\left(\Omega\right) \quad \text{and} \quad \|f^*\|_{W^{\alpha,\infty}\left(\Omega\right)}\leq 1
\end{align*}
for some $\alpha\geq 2$. Given a pre-specified approximation error $\epsilon\in(0,1/\sqrt{d}]$, there exists a neural network model $\tilde{f}\in\cF(L,p)$ with $L = \tilde{O}(\log(1/\epsilon))$ and $p= \tilde{O}(\epsilon^{-\frac{d}{\alpha-1}})$, such that
\begin{align*}
\|\tilde{f} - f^* \|_\infty \leq \epsilon\quad\textrm{and}\quad \|\tilde{f}\|_\mathrm{Lip} \leq 1+\sqrt{d}\epsilon,
\end{align*}
where $\tilde{O}$ hides some negligible constants or log factors.
\end{theorem}

\begin{proof}
Theorem \ref{approx-thm} can be proved based on \cite{guhring2020error}, where under the same condition, they show
\begin{align*}
\norm{\tilde{f}-f}_{W^{1,\infty}(\Omega)} \leq\epsilon.
\end{align*}
Since $f^*\in W^{\alpha,\infty}$ and $\|f^*\|_{W^{\alpha,\infty}\left(\Omega\right)}\leq 1$, we have
\begin{align*}
\norm{\nabla f^*}_2\leq 1\quad\textrm{and}\quad \norm{\nabla \tilde{f}-\nabla f^*}_{\infty}\leq\epsilon,
\end{align*}
Note that though $\tilde{f}$ is using a ReLU activation, $\nabla f$ is well-defined except a measure zero set. Eventually, we obtain
\begin{align*}
\|\tilde{f}\|_\mathrm{Lip}\leq \sup_{\Omega}\norm{\nabla \tilde{f}}_2 \leq  \sup_{\Omega}\norm{\nabla f^* + \nabla \tilde{f} -\nabla f^*}_2 \leq  \sup_{\Omega}\norm{\nabla f^*}_2 + \sqrt{d}\norm{\nabla \tilde{f}-\nabla f^*}_\infty\leq 1+\sqrt{d}\epsilon\leq 2.
\end{align*}
\end{proof}

\section{ERNIE for Mean-Field MARL}
\label{app:wass}
As mentioned in \ref{sec:mean-field}, to learn robust policies we aim to use the regularizer

\begin{align*}
    R^Q_\cW( s; \theta) =
    \mathop{\text{max}}_{\cW(d'_s, d_s)\leq\epsilon} \sum_{a\in \cA} \norm{Q_\theta(s, d'_s, a) - Q_\theta(s, d_s, a)}_2^2.
\end{align*}
However, this optimization problem is difficult to optimize due to the explicit Wasserstein distance constraint. To avoid this computational difficulty, we instead solve the regularized problem

\begin{align*}
    R^Q_\cW( s; \theta) =
    \mathop{\text{max}} \tsum_{a\in \cA} \norm{Q_\theta(s, d'_s, a) - Q_\theta(s, d_s, a)}_2^2 - \lambda_{\cW} \cW(d'_s, d_s).
\end{align*}

The Wasserstein distance term can be computed using IPOT methods with little added computational cost, and we can therefore use this regularizer in a similar manner to the original ERNIE regularizer \citep{xie2020fast}.

\section{Traffic Light Control Implementation Details}
\label{app:traffic}
In our experiments we train four agents in a two by two grid. The length of each road segment is 400 meters and cars enter through each in-flowing lane at a rate of 700 car/hour. The control frequency is 1 Hz, i.e. we need to input an action every second. The reward is based on the following attributes for each agent $n$:
\begin{itemize}
    \item $q^n$: The sum of queue length in all incoming lanes.
    \item $wt^n$: Sum of vehicle waiting time in all incoming lanes.
    \item $dl^n$: The sum of the delay of all vehicles in the incoming lanes.
    \item $em^n$: The number of emergency stops by vehicles in all incoming lanes.
    \item $fl^n$: A Boolean variable indicating whether or not the light phase changed.
    \item $vl^n$: The number of vehicles that passed through the intersection.
\end{itemize}
We can then define the individual reward as
\begin{equation*}
    R^n = -0.5 q^n - 0.5 wt^n - 0.5 dl^n - 0.25 em^n - fl^n + vl^n.
\end{equation*}

All algorithms have the same training strategy. Each agent is trained for five episodes with 3000 SUMO time steps each. At the beginning of training the agent makes random decisions to populate the road network before training begins. Each algorithm is evaluated for 5000 time steps, where the first 1000 seconds are used to randomly populate the road. For adversarial regularization, we use the $\ell_2$ norm to bound the attacks $\delta$.

\subsection{Evaluation Traffic Flows}
The traffic flows used to evaluate the MARL policies in a different environment are shown in table \ref{tab:traffic-flows}. In each flow the total number of cars is similar to the number of cars in the training environment.

\subsection{Details on Changes to Network Topology}
In addition to evaluating traffic light control MARL algorithms when the traffic pattern/speed changes, we also evaluate said MARL algorithms when the traffic network topology slightly changes. We consider two changes to the traffic topology: we slightly change the length of road segments and we evaluate the agents in a larger grid.

To test performance on a larger grid, we evaluate the trained agents on a four by four and six by six traffic light network. Because we only train four agents, we duplicate the trained agents in order to fill out the grid. In the four by four case, we will have four sets of the originally trained four agents, arranged to cover each of the four two by two grids. This setting is especially relevant to the real world deployment of MARL-controlled traffic lights as directly training on a large network may be computationally infeasible.

\begin{table}[!htb]
  \caption{Evaluation Traffic Flows}
  \centering
  \begin{tabular}{rllll}\toprule
    \textit{Flow Number} & Traffic Flow \\ \midrule
    1 & [1000, 1000, 80, 80, 800, 800, 550, 550] \\
    2 & [1000, 1000, 20, 20, 700, 300, 900, 900] \\
    3 & [700, 700, 70, 700, 1400, 600, 80, 80]   \\
    4 & [1000, 1200, 200, 200, 300, 300, 900, 900] \\ 
    5 & [300, 300, 900, 900, 700, 900, 10, 10]\\ 
    \bottomrule
    \label{tab:traffic-flows}
  \end{tabular}
\end{table}

\subsection{Computing resources}
Experiments were run on Intel Xeon 6154 CPUs and Tesla V100 GPUs.

\subsection{Training Details}
Both the actor and critic functions are parametrized by a three-layer multi-layer perceptron with 256 nodes per hidden layer. We use the ADAM optimizer \citep{kingma2014adam} to update parameters and use a grid search to find $\lambda_Q$ and $\lambda_\pi$.

\section{Additional Results}
\label{app:additional}
In this section we include results that we could not fit in the main paper due to limited space. In particular, we show an evaluation of COMA's robustness with the ERNIE framework and some additional environment changes.

\subsection{Multi-Agent Drone Control}
\label{app:drone}
To evaluate the performance of ERNIE in multi-agent robotics environments, we use the multi-agent drone environment \citep{panerati2021learning}. We find that ERNIE can indeed provide enhanced robustness against input perturbations. The results can be found in Figure \ref{fig:drone}.
\begin{figure*}
    \centering
    \begin{subfigure}{0.3\textwidth}
        \centering
        \includegraphics[height=1.5in]{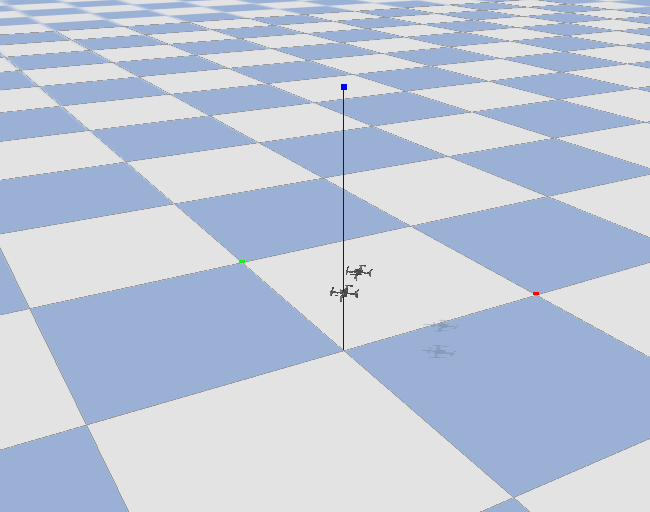}
        \caption{Environment Image}
    \end{subfigure}%
    \begin{subfigure}{0.3\textwidth}
        \centering
        \includegraphics[height=1.5in]{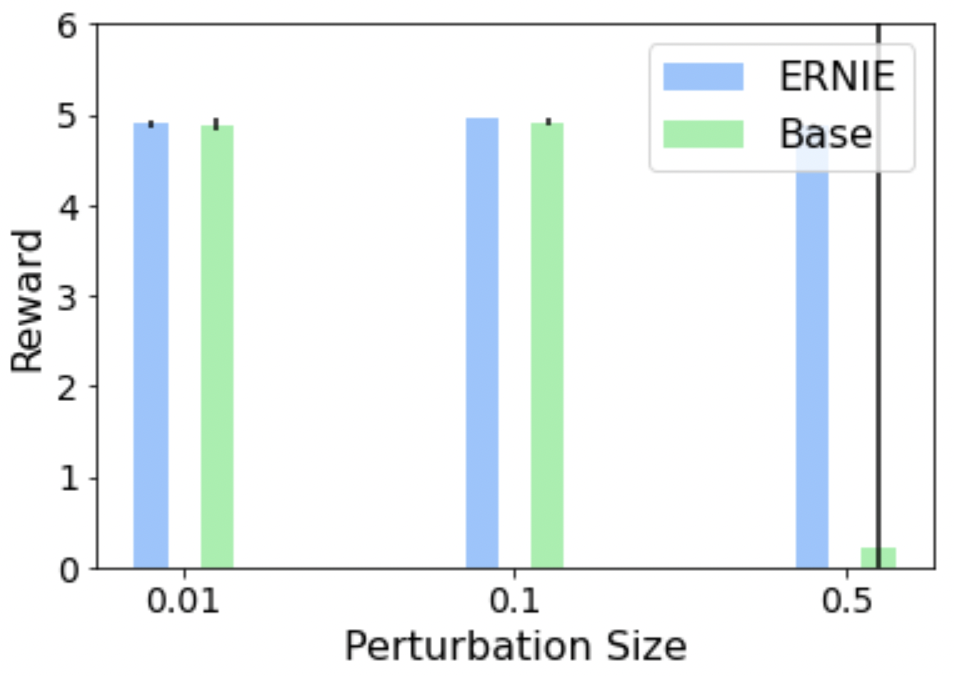}
        \caption{Agent 1 Robustness}
    \end{subfigure}%
    \begin{subfigure}{0.3\textwidth}
        \centering
        \includegraphics[height=1.5in]{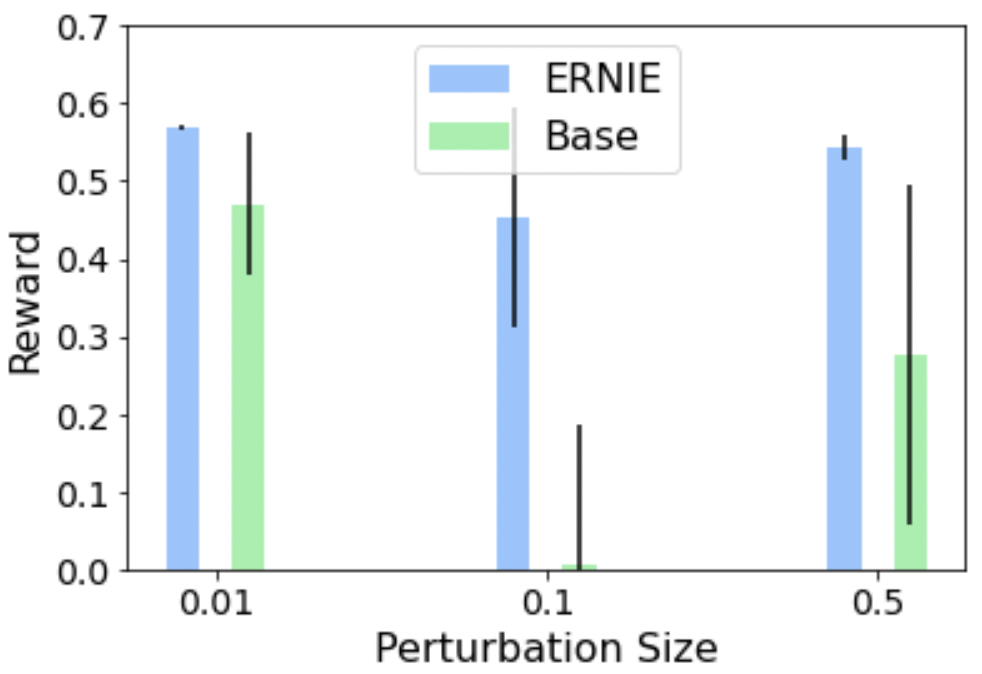}
        \caption{Agent 2 Robustness}
    \end{subfigure}%
    \caption{Evaluation of ERNIE in the multi-agent drone environment (see Figure 3a). The baseline algorithm we use is MAPPO. We then perturb the observation of each of the two agents with Gaussian noise to evaluate robustness (see Figure 3b-c). The task is follow the leader, where the agents have to navigate while remaining close to each other.}
    \label{fig:drone}
\end{figure*}

\subsection{ERNIE for COMA (Traffic Light Control)}
We apply ERNIE to improve the robustness COMA for traffic light control. Figure \ref{fig:COMA} shows the performance of COMA with and without ERNIE on various environment changes. From Figure \ref{fig:COMA} we can see that the ERNIE and ERNIE w/o ST frameworks are able to outperform the baseline in all of the perturbed environments, indicating increased robustness. From table \ref{tab:coma}, we can again see that the ERNIE framework provides increased robustness to every environment change. Interestingly, ERNIE outperforms ERNIE w/0 ST in the training environment and in the setting with small amounts of observation noise (see Figure \ref{fig:COMA}), suggesting that the Stackelberg formulation allows for a better fit to the lightly perturbed data than conventional adversarial training does.  

\begin{figure}[!htb]
\centering
     \begin{subfigure}{0.4\textwidth}
         \centering
         \includegraphics[width=\textwidth]{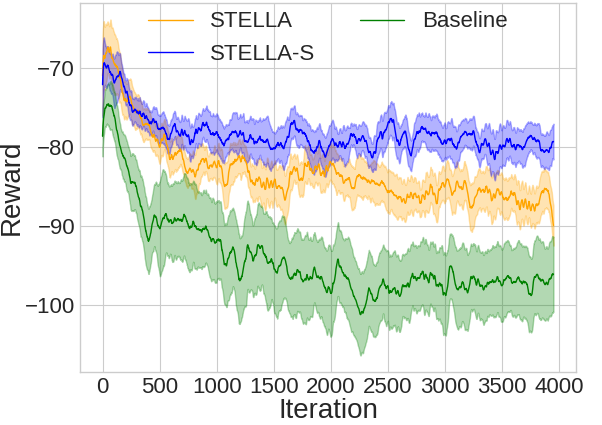}
         \caption{Different Speed (40 m/s)}
     \end{subfigure}
     \begin{subfigure}{0.4\textwidth}
         \centering
         \includegraphics[width=\textwidth]{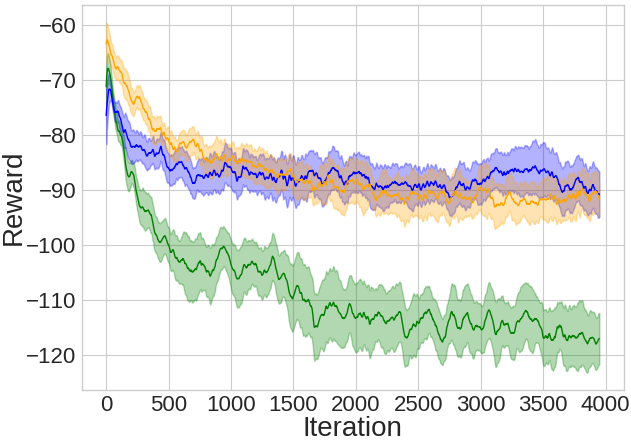}
         \caption{Gaussian Noise (0.01)}
     \end{subfigure}
    \begin{subfigure}{0.4\textwidth}
         \centering
         \includegraphics[width=\textwidth]{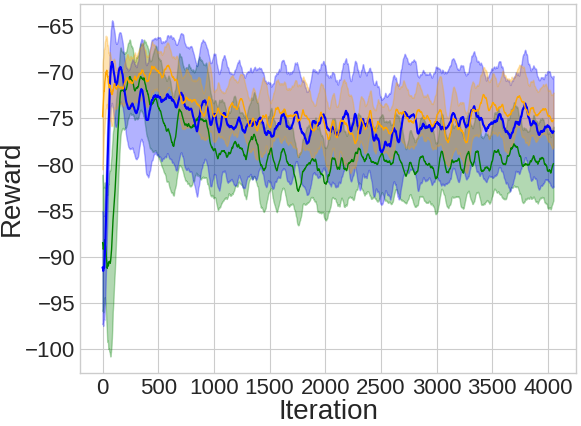}
         \caption{Different Traffic Flow (Flow-3)}
     \end{subfigure}
     \begin{subfigure}{0.4\textwidth}
         \centering
         \includegraphics[width=\textwidth]{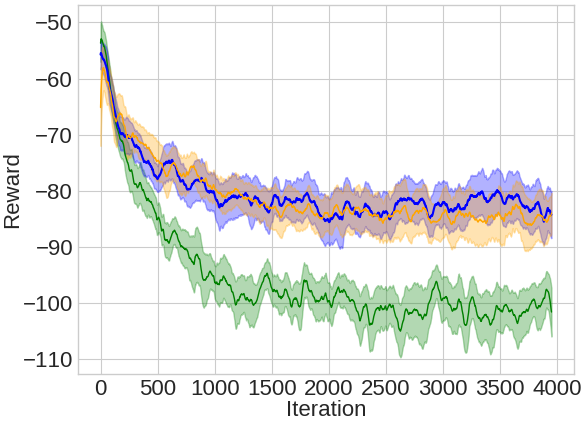}
         \caption{Irregular Grid Topology (COMA)}
     \end{subfigure}
     \begin{subfigure}{0.42\textwidth}
         \centering
         \includegraphics[width=\textwidth]{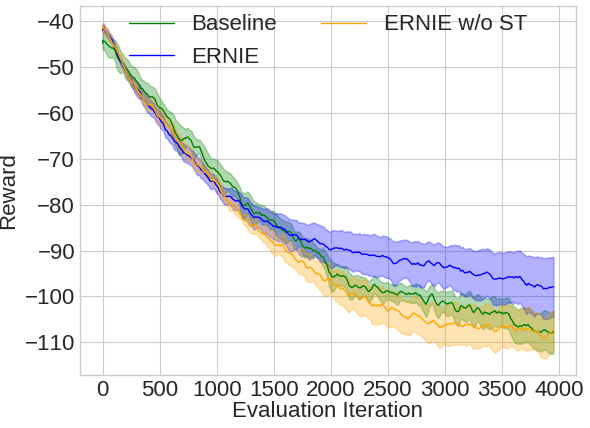}
         \caption{Larger Grid Topology (COMA)}
     \end{subfigure}
    \caption{Evaluation curves from COMA on different environment changes for traffic light control.}
    \label{fig:COMA}
\end{figure}

\begin{table*}[!htb]
\centering
  \caption{Evaluation rewards and standard deviation  for the traffic light control task under different environment perturbations. The baseline algorithm is COMA.}
  \label{tab:coma}
  \begin{tabular}{rccccc}\toprule
    \textit{Algorithm} & \textit{Train} & \textit{Obs. Noise (0.1)} & \textit{Obs. Noise (1.0)} & \textit{Speed (30 m/s)} \\ \midrule
    ERNIE & $\mathbf{-78.39(3.12)}$ & $-86.37(7.1)$ & $-102.45(3.45)$ & $-91.73(6.84)$ \\
    ERNIE w/o ST & $-83.07(3.52)$ & $\mathbf{-84.66(6.15)}$ & $\mathbf{-101.37(3.05)}$ & $\mathbf{-91.69(8.82)}$ \\
    Baseline & $-93.97(7.34)$& $-108.05(8.68)$ & $-113.24(5.82)$ & $-108.45(3.98)$\\
    \bottomrule
    \label{tab:coma}
  \end{tabular}
\end{table*}

\subsection{Additional Results on Changed Networks}
In addition to evaluating the performance of ERNIE and the baseline algorithms on the $4 \times 4$ network, we evaluate the performance of these algorithms on a $6 \times 6$ network. The results shown in table \ref{tab:large-eval} shows that ERNIE and ERNIE-S again outperform the baseline algorithm in the changed environment, indicating increased robustness.

\begin{table}[htb!]
    \centering
  \caption{Evaluation rewards and standard deviations on larger networks.}
  \begin{tabular}{rcc}\toprule
    \textit{Algorithm} & $\mathit{4 \times 4}$ & $\mathit{6 \times 6}$ \\ \midrule
    Baseline (QCOMBO) & $-401.64(22.25)$ & $-320.66(40.80)$ \\
    ERNIE w/o ST & $-221.24(13.88)$ & $-213.20(14.04)$ \\
    ERNIE & $\mathbf{-217.21(8.36)}$ & $\mathbf{-152.60(3.91)}$ \\ 
    \hline
    Baseline (COMA) & $-384.17$ &$-330.55(5.70)$   \\
    ERNIE  & $-394.14(1.29)$ &$-337.25(3.86)$    \\ 
    ERNIE w/o ST & $\mathbf{-369.40(6.04)}$ & $\mathbf{-319.16(3.95)}$ \\ 
    \bottomrule
    \label{tab:large-eval}
  \end{tabular}
\end{table}

We also evaluate the performance of ERNIE in another irregular traffic network from Atlanta. This grid can be see in Figure \ref{fig:atl}, and the performance of ERNIE and the baselines can be seen in table \ref{tab:real}. As with the other environment changes, we can see that the ERNIE framework exhibits increased robustness over the baseline algorithms.
\begin{figure}
    \centering
    \includegraphics[width=0.3\textwidth]{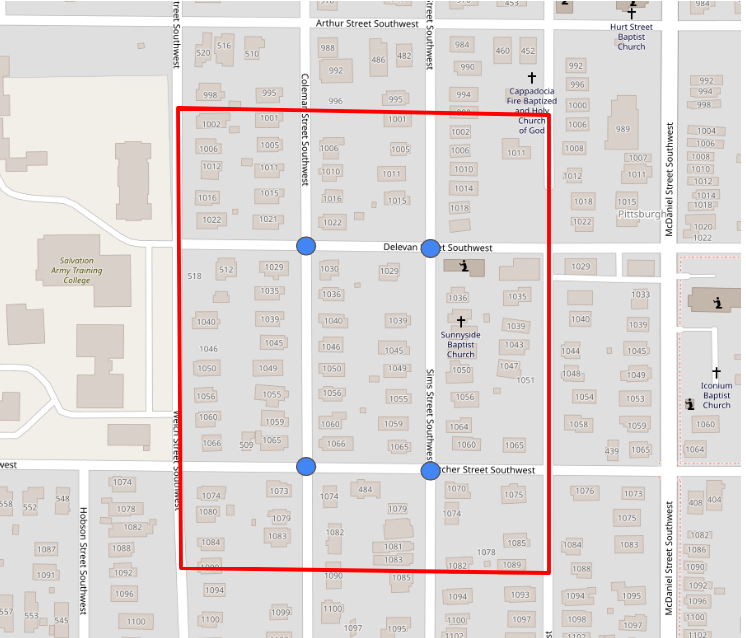}
    \caption{Irregular $2 \times 2$ traffic network from Atlanta.}
    \label{fig:atl}
\end{figure}

\begin{table}[htb!]
\centering
  \caption{Evaluation rewards and standard deviation on irregular networks}
  \label{tab:real}
  \begin{tabular}{rc}\toprule
    \textit{Algorithm} & \textit{Atlanta}  \\ \midrule
    Baseline (QCOMBO) &  $-435.69(27.09)$ \\
    ERNIE w/o ST & $-339.48(28.98)$   \\
    ERNIE  & $\mathbf{-285.84(28.44)}$ \\
    \hline
    Baseline (COMA) &  $-477.54(4.41)$   \\
    ERNIE w/o ST & $\mathbf{-402.12(5.67)}$ \\ 
    ERNIE & $-432.87(5.43)$  \\ 
    \bottomrule
  \end{tabular}
\end{table}

\subsection{Additional Ablation Experiments}
\label{app:able}
To further verify the effectiveness of the Stackelberg reformulation of adversarial regularization, we compare the performance of ERNIE with and without ST (Stackleberg Training) in the particle environments. The results are shown in Figure \ref{fig:particle-able}, where we can see that the Stackelberg formulation performs better or equivalently to normal adversarial regularization in all settings.

\begin{figure*}[htb!]
\centering
     \begin{subfigure}{0.22\textwidth}
         \centering
         \includegraphics[width=\textwidth]{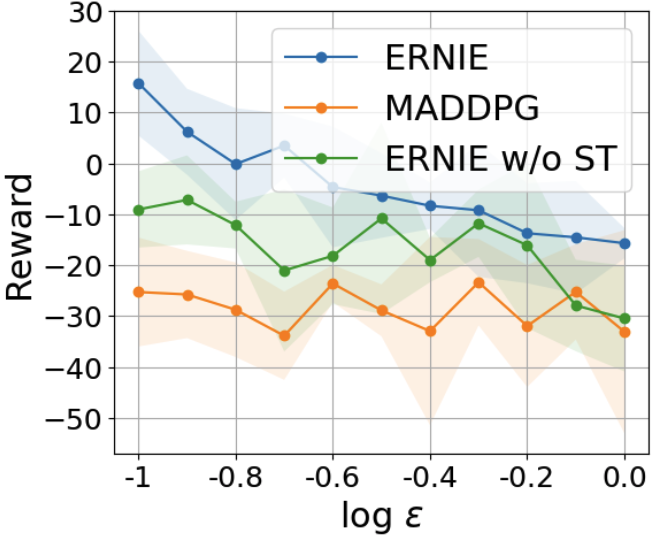}
         \caption{Covert Comm.}
     \end{subfigure}
     \hspace{0.001\textwidth}
     \begin{subfigure}{0.215\textwidth}
         \centering
         \includegraphics[width=\textwidth]{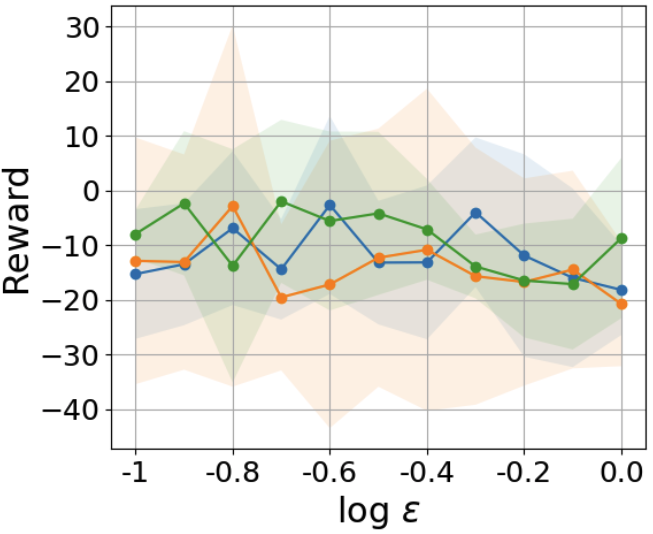}
           \caption{Tag}
     \end{subfigure}
     \hspace{0.001\textwidth}
    \begin{subfigure}{0.22\textwidth}
         \centering         \includegraphics[width=\textwidth]{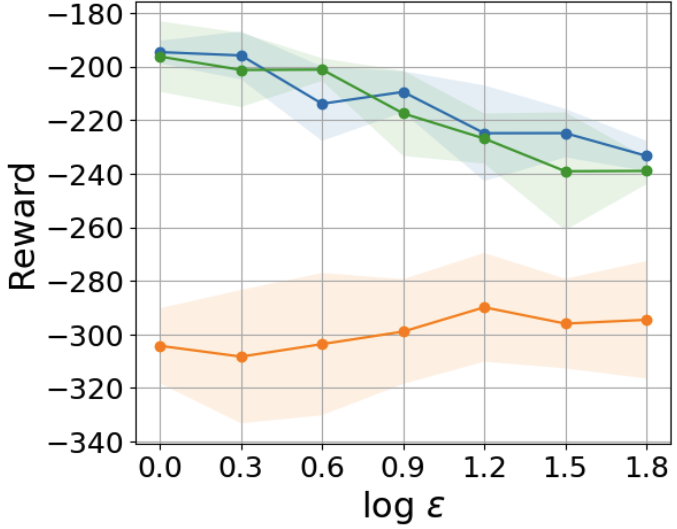}
           \caption{Navigation}
     \end{subfigure}
     \hspace{0.001\textwidth}
    \begin{subfigure}{0.225\textwidth}
         \centering         \includegraphics[width=\textwidth]{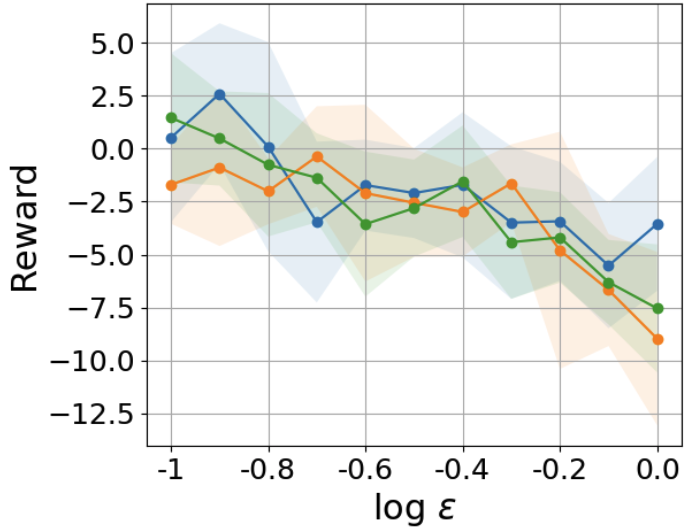}
           \caption{Predator Prey}
     \end{subfigure}
    \caption{Ablation study comparing ERNIE with and without Stackelberg Training (ST).}
    \label{fig:particle-able}
\end{figure*}

\subsection{Time Comparison}
In the cooperative navigation environment with 3 agents, we find that the baseline
MADDPG takes 1.127 seconds for 50 episodes, ERNIE takes 1.829 seconds, and M3DDPG takes 3.250 seconds.
Although ERNIE is more expensive than vanilla training, it is significantly more efficient than competitive baselines.

\section{Baseline Algorithms}
\label{app:baselines}
In this section we describe the baseline algorithms in detail.

\subsection{QCOMBO}
QCOMBO \citep{zhang2019integrating} is a Q-learning based MARL algorithm that couples independent and centralized learning with a novel regularization method. QCOMBO consists of three components, an {\it individual part}, a {\it global part}, and a {\it consistency regularization}. 
The individual part consists of  Q-learning for each agent 
\begin{equation*}
    \mathcal{L}(\theta) = \frac{1}{N} \sum_{n=1}^N \mathbb{E} \left[\frac{1}{2} (y^n_t - Q^n(o^n_t, a^n_t; \theta^n))^2\right],
\end{equation*}
where
$
    y^n_t = r^n_t + \gamma \mathop{\text{max}}_{\hat a^n} Q^n(o^n_{t+1}, \hat a^n; \theta^n), ~~ \forall n \in [N] ,
$
and $\theta = [\theta^1, \ldots, \theta^n]$ denotes the concatenation of local parameters.

The global part consists of a global Q-network that learns a global Q function. 
We parameterize the global Q-function by $\omega$, and minimize the approximate Bellman residual
\begin{align}
\label{bellman}
    \mathcal{L}(\omega) = \mathbb{E} \big[\frac{1}{2} (y_t - Q(s_t, \mathbf{a}_t; \omega))^2\big],
\end{align}
where $y_t = r^g_t + \gamma  Q(s_{t+1}, \mathbf{a}'_t;\omega)$ and 
$\ab'_t = (a_1^t, \ldots, a_N^t),~a_n^t \in \argmax_{\hat{a}^n \in A^n } Q^n(o^n_{t+1}, \hat a^n; \theta^n).$
Finally a consistency regularization 
$$
    \mathcal{L}_{\mathrm{reg}}(\omega, \theta) = \mathbb{E} \big[\frac{1}{2} (Q(s, \mathbf{a}; \omega) -\sum_{n=1}^N Q^n(o^n, a^n; \theta^n))^2\big]
$$
ensures that global and individual utility functions are similar, to encourage cooperation. The complete QCOMBO loss is then given by
\begin{equation*}
    \mathcal{L}_{\mathrm{QC}}(\omega, \theta) = \mathcal{L}(\omega) + \mathcal{L}(\theta) + \lambda_Q \mathcal{L}_{\text{reg}}(\omega, \theta).
\end{equation*}
Here $\lambda_Q$ is a hyperparameter that can be tuned. In execution decisions are made with the individual utility functions,
$\{Q^n\}_{n=1}^N$. In practice we apply ERNIE to the individual Q-functions $Q^n$.

\subsection{MADDPG}
MADDPG is a multi-agent version of Deep Deterministic Policy Gradient (DDPG). DDPG uses the actor-critic architecture where a state-action value function $Q_\phi$ is used to update a deterministic policy $\mu_\theta$. The state-action value function is updated to minimize the squared bellman loss 
\begin{equation*}
    \cL(\phi) = \mathbb{E}_{s_t \sim \rho}[(Q_\phi(s_t, a_t) - y_t)^2]
\end{equation*}
where $y_t = r_t + Q'(s_{t+1}, \mu_\theta'(s_{t+1}))$ and $Q_\phi'(\cdot), \mu_\theta'(\cdot)$ are target networks. The policy function is updated with the policy gradient taking the form
\begin{equation*}
    \mathbb{E}_{s_t \sim \rho} \big[ \nabla Q_\phi(s_t, a_t)|_{a_t=\mu_\theta(s_t)} \nabla \mu_\theta(s_t)\big].
\end{equation*}
The target networks are gradually updated throughout training to track the actor and critic networks.

MADDPG extends DDPG to the multi-agent setting with the paradigm of centralized training with decentralized execution. In particular, MADDPG employs a centralized state-action value function $Q_\phi$ and independent actor functions $\{\mu_{\theta_1}, ..., \mu_{\theta_N}\}$. Denoting $\mathbf{a_t}$ as the joint action of the agents at time $t$, the state-action value function is updated to minimize the squared bellman loss
\begin{equation*}
    \cL(\phi) = \mathbb{E}_{s_t \sim \rho}[(Q_\phi(s_t, \mathbf{a_t}) - y_t)^2]
\end{equation*}
where $y_t = r_t + Q'(s_{t+1}, \mu_{\theta_1}'(o_{1,t+1}), ..., \mu_{\theta_N}'(o_{N,t+1}))$ and $Q_\phi'(\cdot), \mu_{\theta_1}'(\cdot), ..., \mu_{\theta_N}'(\cdot)$ are target networks. Each policy function $\mu_{\theta_i}$ is updated with the policy gradient
\begin{equation*}
 \mathbb{E}_{s \sim \rho} \big[ \nabla Q_\phi(s, \mathbf{a})|_{\mathbf{a}=\mu_{\theta_1}(o_1), ..., \mu_{\theta_N}(o_N)} \nabla \mu_{\theta_i}(o_i)\big].
\end{equation*}
where $\rho$ is the state visitation distribution. Note that the state-action value function is only used during training and that actions are only taken with the decentralized policy functions. In practice we apply ERNIE to the individual policies $\mu_\theta$.

\subsection{COMA}
COMA is a policy gradient algorithm that
directly seeks to minimize the negative cumulative reward $\cL_{\mathrm{NCR}}$ by 
learning
$\{\pi_n\}_{n=1}^N$ parametrized by $\theta = \cbr{\theta_n}_{n=1}^N$ with the actor-critic training paradigm. 
Specifically, COMA updates local policies (actors) with policy gradient
\begin{align}
\label{eq:coma-grad}
    \nabla L_{\mathrm{NCR}}(\theta) = \mathbb{E}_\pi\big[\sum_{n=1}^N \nabla_\theta \text{log } \pi^n(a^n|o^n)A^n(s, \mathbf{a})\big],
\end{align}
where $A^n(s, \mathbf{a})$ is the counterfactual baseline given by 
$
    A^n(s, \mathbf{a}) = Q(s, \mathbf{a}) - \sum_{\tilde{a}^n} \pi^n(\tilde{a}^n|o^n) Q(s, (\mathbf{a}^{-n}, \tilde{a}^n)).
    $
The critic parametrized by $\theta^c$ is with trained with
$
    \mathcal{L}(\theta^c) =   \mathbb{E}_\pi \big[\frac{1}{2} (y_t - Q_{\theta^c}(s_t, \mathbf{a}_t))^2\big],
    $
where $y^n_t$ is the target value defined in TD($\lambda)$ \cite{sutton2018reinforcement}.
In execution decisions are made with the individual policy functions $\{\pi^n\}_{n=1}^N$. In practice we apply ERNIE to the individual policies $\pi^n$.

\section{Particle Environments Implementation Details}
\label{app:traffic}
For the particle environments task, we follow the implementation of \href{https://github.com/shariqiqbal2810/maddpg-pytorch}{maddpg-pytorch}. For each task we parametrize the policy function with a three layer neural network, with 64 units hidden units. We then train for 25000 epochs (covert communication) 15000 epochs (cooperative navigation and predator prey), or 5000 epochs (tag). As we are considering the cooperative setting, we only apply ERNIE to the cooperating agents. The reward in the perturbed environments is that of the cooperative agents (note that we do not perturb the observations of the opposition agent). For adversarial regularization, we use the $\ell_2$ norm to bound the attacks $\delta$. We use SGD to update parameters and use a grid search to find and $\lambda_\pi$.

\subsection{Mean-Field Implementation}
For our mean-field implementation, we use the implementation of \citet{li2021permutation}. For $N=3,30$ agents we use a batch size of 32. For $N=6, 15$, we use a batch size of 100. We train for 10000 episodes, and use a replay buffer of size 100. All other configurations should be the same as used in \citet{li2021permutation}.

\subsection{M3DDPG Implementation}
We implement our own version of M3DDPG in PyTorch \citep{paszke2019pytorch}, as the original implementation uses Tensorflow \citep{abadi2016tensorflow}. In each setting, we tune the attack steps size $\epsilon \in [1e-5, 1e-2]$.

\section{ERNIE-A}
\label{app:action}
We show our algorithm for solving \eqref{eq:action-reg}. Note that $\mathbf{a'} \cup \mathbf{a}_{i, j}$ refers to the joint action $\mathbf{a}$ where the action of agent $i$ is changed to $j$.


        
            
\begin{figure*}[htb!]
\includegraphics[width=\textwidth]{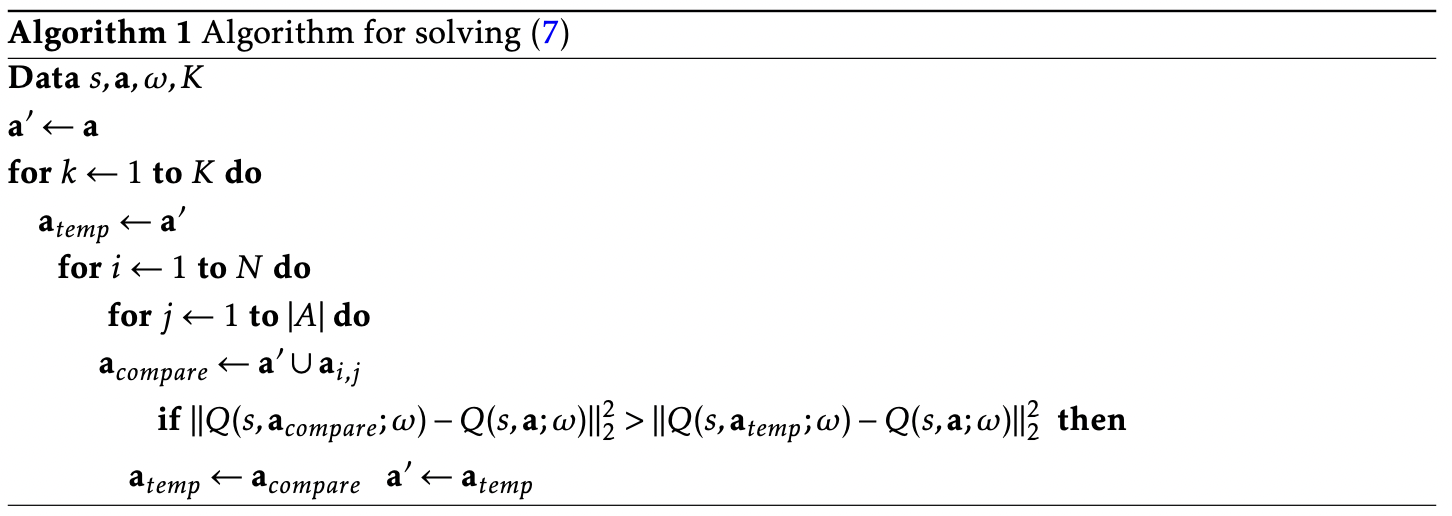}
\captionlistentry{}
\label{alg}
\end{figure*}

\section{Gaussian Baseline}
\label{app:gaussian}
The baseline-Gaussian is similar to ERNIE. However, instead of generating $\delta$ as
\begin{align*}
    \delta = \mathop{\text{argmax}}_{||\delta|| \leq \epsilon} D(\pi_{\theta_k}(o_k + \delta), \pi_{\theta_k}(o_k)),
\end{align*}
$\delta$ is sampled from the standard normal $\mathcal{N}(0, I)$. Similar to ERNIE, this baseline will ensure the policy does not change to much given Gaussian input perturbations. This baseline therefore performs well in several environments, especially those with Gaussian observation noise. However, robustness against Gaussian noise does not ensure robustness against all noise, and the Gaussian baseline may therefore fail in some perturbed environments.

\end{document}

%% file: source_files/intro.tex

In the past decade advances in deep neural networks and greater computational power have led to great successes for Multi-Agent Reinforcement Learning (MARL), which has achieved success on a wide variety of multi-agent decision-making tasks ranging from traffic light control \citep{wiering2000multi} to StarCraft \citep{vinyals2019grandmaster}. However, while much effort has been devoted to applying MARL to new problems, there has been limited work regarding the robustness of MARL policies. 

Despite the limited attention paid to robustness, it is essential for MARL policies to be robust. Most MARL policies are trained in a fixed environment. Since these policies are trained solely to perform well in that environment, they may perform poorly in an environment with slightly different transition dynamics than the training environment. In addition, while agents are fed with exact state information in training, MARL policies deployed in the real world can receive inaccurate state information (e.g., due to sensor error). Finally, even a single agent acting maliciously or differently than expected can cause a chain reaction that destabilizes the whole system. These phenomena cause significant concern for the real-world deployment of MARL algorithms, where the environment dynamics and observation noise can change over time. 
We observe that even when the change in the environment's dynamics is small, the performance of MARL algorithms can deteriorate severely (See an example in Section \ref{sec:experiments}).  Thus there is an emerging need for MARL algorithms that are robust to changing transition dynamics, observation noise, and changing behavior of agents.

Although many robust RL methods have been proposed for the single agent case, three major barriers prevent their use for MARL. Theoretically, it is not clear if or when such methods can work for MARL. Methodologically, it is not straightforward to apply single agent robust RL methods to MARL, as single agent methods may not consider the interactions between several agents. Algorithmically, single agent robust RL algorithms are often unstable, and may not perform well when applied to inherently unstable MARL training.
Therefore to learn robust MARL policies, we provide theoretical, methodological, and algorithmic contributions. 

\textbf{Theory.} Theoretically, we first show that when the transition and reward function are smooth, a policy's value function is also smooth. In our experiments, we show that this assumption can serve as a useful prior knowledge, even if the transition function is not smooth in every state. Second, we prove that a smooth and close-to-optimal policy exists in any such environment. Third, we show that a policy's robustness is inversely proportional to its Lipchitz constant \emph{with no smoothness assumption on the environment's smoothness}. 
These observations advocate for using smoothness as an inductive bias to not only reduce the policy search space, but simultaneously improve the robustness of the learned policy.
Finally, we prove that large neural networks are capable of approximating the target policy or Q functions with smoothness guarantees. These findings give us the key insight that in order to learn robust and high-performing deep MARL policies, we should enforce the policies' smoothness.

\textbf{Method.} Based on these findings, we propose a new framework -- adv\textbf{E}rsarially \textbf{R}egularized multiage\textbf{N}t re\textbf{I}nforcement l\textbf{E}arning (ERNIE), that applies adversarial training to learn smooth and robust MARL policies in a principled manner. In particular, we develop an adversarial regularizer to minimize the discrepancy between each policy's output given a perturbed observation and a non-perturbed observation. This adversarial regularization gives two main benefits: 
Lipschitz continuity and rich data augmentation with adversarial examples.
The adversarial regularization encourages the learned policies to be Lipschitz continuous, improving robustness. 
Augmenting the data with adversarial examples further provides robustness against environment changes.
Specifically, new scenarios emerge when the environment changes,  
and data augmentation with adversarial examples provides a large coverage of these scenarios as long as the environment change is small.
Adapting to adversarial examples during training ensures that the agents will perform reasonably even in the worst case.

To further provide robustness against the changing behaviors of a few malicious agents, we propose an extension of ERNIE that minimizes the discrepancy between the global Q-function with maliciously perturbed joint actions and non-perturbed joint actions. This regularizer encourages the policies to produce stable outputs even when a subset of agents acts sub-optimally, therefore granting robustness. Such robustness has not been considered in previous works.

\textbf{Algorithm.} We find that adversarial regularization can improve robust performance \citep{shen2020deep}. However, adversarial regularization can also be unstable. More concretely, conventional adversarial regularization can be formulated as a zero-sum game where the defender (the policy) and attacker (the perturbation) hold equal positions and play against each other. In this case, a small change in the attacker's strategy may result in a large change for the defender, rendering the problem ill-conditioned. Coupled with the already existing stability issues that come with training MARL algorithms, this instability issue greatly reduces the power of adversarial regularization methods for MARL. 

To address this issue, we reformulate adversarial training as a Stackelberg game. In a Stackelberg game, the leader (defender) has the advantage as it knows how the follower (attacker) will react to its actions and can act accordingly. This advantage essentially makes the optimization problem smoother for the defender, leading to a more stable training process. 

\textbf{Extension to Mean-field MARL.} We further demonstrate the general applicability of ERNIE by developing its extension to robustify mean-field MARL algorithms. 
The mean-field approximation has been widely received as a practical strategy to scale up MARL algorithms 
while avoiding the curse 
of many agents 
\citep{wang2020breaking}.
However, as mean-field algorithms are applied to real-world problems, it is essential to develop robust versions.
To facilitate policy learning that is more robust, we extend ERNIE to mean-field MARL with a formulation based on distributionally robust optimization \citep{delage2010distributionally, goh2010distributionally}. 

To demonstrate the effectiveness of the proposed framework, we conduct extensive experiments that evaluate the robustness of ERNIE on traffic light control and particle environment tasks. Specifically, we evaluate the robustness of MARL policies when the evaluation environment deviates from the training environment. These deviations include observation noise, changing transition dynamics, and malicious agent actions. The results show that while state-of-the-art MARL algorithms are sensitive to small changes in their environment, the ERNIE framework enhances the robustness of these algorithms without sacrificing efficiency.

\textbf{Contributions. } We remark that adversarial regularization has been developed for single-agent RL, but never for MARL \citep{shen2020deep}. Our contribution in this paper has four aspects: (1) advances in theoretical understanding (2) development of new regularizers for MARL (3) new algorithms for stable adversarial regularization in MARL (4) comprehensive experiments in a number of environments.

%% file: source_files/background.tex
In this section, we introduce the necessary background for MARL problems together with related literature. 
We consider the setting of \emph{cooperative MARL}, where agents work together to maximize a global reward.

$\bullet$ \noindent \textbf{Cooperative Markov Games}.
We consider a partially observable Markov game $\langle \mathcal{S}, \mathcal{O}^N, \mathcal{A}^N, \mathcal{P}, \mathcal{R},\\ N, \gamma \rangle$ in which a set of agents interact within a common environment. We let $\mathcal{S} \subseteq \mathbb{R}^S$ denote the global state space, $\mathcal{O} \subseteq \mathbb{R}^O$ denote the observation space for each of the $N$ agents, $\mathcal{A} \subseteq \mathbb{R}^A$ denote the action space, $\mathcal{P}:\mathcal{S} \times \mathcal{A} \mapsto \mathcal{S}$ denote the transition kernel, $\gamma$ denotes the discount factor, and $\mathcal{R}: \mathcal{S} \times \mathcal{A} \mapsto \mathbb{R}^N$ denotes the reward function.
At every time step $t$, each of the $N$ agents selects an action according to its policy, which can be stochastic or deterministic. Then, the system transitions to the next state according to the transition kernel and each agent receives a reward $r_{i, t}$. We denote the global reward at time $t$ as $r^g_t$. The goal of each agent is to find a policy that maximizes the discounted sum of its own reward, $\tsum_{t \geq 0} \gamma^t r_{i, t}$. 

$\bullet$ \noindent \textbf{Robust RL}. In recent years many single agent robust RL techniques have been proposed. Most of these methods use information about the underlying simulator to train agents over a variety of relevant environment settings \citep{morimoto2005robust, pinto2017robust, abdullah2019wasserstein, pan2019risk, wang2022policy}. Although these methods can provide robustness against a wide range of environment changes, they suffer from long training times and require expert knowledge of the underlying simulator, which is not practical. Another direction of research focuses on perturbation based methods \citep{shen2020deep, zhang2020robust}. Perturbation based methods train the policy to be robust to input perturbations, encouraging the policy to act reasonably in perturbed or previously unseen states. \cite{kumar2021policy} certify robustness by adding smoothing noise to the state; it is not clear how this affects the learned policy's optimality. Another related line of work \citep{iyengar2005robust, nilim2005robust, panaganti2022robust, li2022first} studies robust markov decision processes and provides a principled way to learn robust policies. However, such methods often require strict assumptions on the perturbation/uncertainty. Inspiring our work, \cite{shen2020deep} proposes to learn a smooth policy in single agent RL, but they do so to reduce training complexity rather than increase robustness and provide no theoretical justification for their method. Instead, we theoretically connect smoothness to robustness, extend perturbation based methods to MARL, and develop a more stable perturbation computation technique, and develop an extension to mean-field MARL.


$\bullet$ \noindent \textbf{Robust MARL}. Recently, some works have studied the robustness of MARL systems. \citet{lin2020robustness}
studies how to attack MARL systems and finds that MARL systems are vulnerable to attacks on even a single agent. \citet{zhang2021robust} develop a framework to handle MARL with model uncertainty by formulating MARL as a robust Markov game. However, their proposed method only considers uncertainty in the reward function, while this article focuses on robustness to observation noise and changing transition dynamics. \citet{li2019robust} modify the MADDPG algorithm to consider the worst-case actions of the other agents in continuous action spaces with the M3DDPG algorithm. M3DDPG aims to grant robustness against the actions of other agents, which is less general than the robustness against observation noise, changing transition dynamics, and malicious agents that our method aims for. \citet{wang2022data} consider robustness against uncertain transition dynamics, but their algorithm is not applied to deep MARL. More recently, \citet{he2023robust, han2022solution} introduces the concept of robust equilibrium and proposes to learn an adversarial policy to perturb each agent's observations. Finally \citet{zhou2023robust} propose to learn robust policies by minimizing the cross-entropy loss between agent's actions in non-perturbed states and perturbed states.

The ERNIE framework is also related to several existing works which use similar adversarial training methods but target different domains such as trajectory optimization~\citep{zhao2021adversarially}, semi-supervised learning~\citep{miyato2018virtual,hendrycks2019using,zuo-etal-2022-self}, fine-tuning language models~\citep{jiang2019smart,yu-etal-2021-fine}, and generalization in supervised learning~\citep{zuo2021adversarial}.

%% file: source_files/theory.tex
This section presents the theoretical motivation for our algorithm by showing that Lipschitzness (smoothness) serves as a natural way to gain robustness, while reducing the policy search space.
We start by observing that certain natural environments exhibit smooth transition and reward functions, especially when the transition dynamics are governed by physical laws (e.g., MuJuCo environment \citep{todorov2012mujoco}, Pendulum \citep{1606.01540}).\footnote{
See Remark \ref{remark_approx_smooth} for discussions when the smoothness property holds approximately.
}
Formally, this is stated as the following.

\begin{definition}
Let $\cS \subseteq \RR^d$.
We say the environment is $(L_r, L_{\PP})$-smooth, if the
 reward function $r: \cS \times \cA \to \RR$, 
and the transition kernel $\PP: \cS \times \cS \times \cR$ satisfy
\begin{align*}
\abs{r(s,a) - r(s',a)} &\leq L_r \norm{s - s'} ~~\text{and}~~
\norm{\PP(\cdot|s,a) - \PP(\cdot|s',a)}_1 \leq L_{\PP} \norm{s - s'},
\end{align*}
for $(s,s',a) \in \cS \times \cS \times \cA$. $\norm{\cdot}$ denotes a metric on $\RR^d$. 
We say a policy $\pi$ is $L_\pi$-smooth if 
\begin{align*}
\norm{\pi(\cdot|s) - \pi(\cdot|s')}_1 \leq L_\pi \norm{s - s'}.
\end{align*}
\end{definition}

Without loss of generality, we assume $\abs{r(s,a)} \leq 1 $ for  any $(s,a) \in \cS \times \cA$.
We then present our theory. Due to the space limit, we defer all technical details to the appendix.

{\bf $\bullet$ From smooth environments to smooth values.}
 We proceed to show that if the environment is smooth, then the value functions for smooth policies are also smooth.

\begin{theorem}\label{theorem_smooth_vals}
Suppose the environment is $(L_r, L_\PP)$-smooth.
Then the Q-function of any policy $\pi$, defined as 
\begin{align*}
Q^{\pi}(s,a) = \EE^{\pi} \sbr{\tsum_{t=0}^\infty \gamma^t r(s_t,a_t) | s_0 = s, a_0=a},
~ \forall (s,a) ,
\end{align*}
is Lipschitz continuous in the first argument. 
That is, 
\begin{align}\label{eq_smooth_q}
\abs{Q(s,a) - Q(s',a)}  \leq L_Q \norm{s - s'}, 
\end{align}
where $L_Q \coloneqq L_r + \gamma   L_\PP / (1-\gamma)$. Suppose in addition the policy is $L_\pi$-smooth.
Then the value function, defined as 
\begin{align*}
V^{\pi}(s) 
= \EE^{\pi} \sbr{\tsum_{t=0}^\infty \gamma^t r(s_t,a_t) | s_0 = s},
~ \forall s,
\end{align*}
is Lipschitz continuous.
That is,
\begin{align*}
\abs{V^{\pi}(s) - V^{\pi}(s')}
\leq L_V \norm{s - s'},
\end{align*}
where $L_V \coloneqq  L_{\pi} /(1-\gamma) + L_Q$.
\end{theorem}
In view of Theorem \ref{theorem_smooth_vals}, it is clear that whenever the environment and the policy are smooth, then the value functions are also smooth.
A natural and important follow-up question to ask is 
whether this claim holds in the reverse direction.
More concretely, we ask whether it is reasonable to seek a policy that is also smooth with respect to the state while maximizing the reward. 
If the claim holds true, then seeking a smooth policy can serve as an efficient and unbiased prior knowledge, that can help us reduce the policy search space significantly, while still guaranteeing that we are searching for high-performing policies.

{\bf $\bullet$ Existence of smooth and nearly-optimal policies.}
The following result shows that for any $\epsilon > 0$,
there exists an $\epsilon$-optimal policy that is $\cO(L_Q/\epsilon)$ smooth, where $L_Q$ defined in Theorem \ref{theorem_smooth_vals} only depends on the smoothness of reward and transition.
This structural observation naturally suggests seeking a smooth policy for smooth environments.

\begin{theorem}\label{thrm_existence_smooth_policy}
Suppose the environment is $(L_r, L_\PP)$-smooth.
Then for any $\epsilon > 0$,
there exists an $\epsilon$-optimal policy $\pi$ that is also smooth, i.e.,  
\begin{align*}
    V^*(s) - V^\pi(s) \leq \tfrac{2\epsilon}{1-\gamma}, ~~ \forall s \in \cS ~~\textrm{and}~~
    \norm{\pi(\cdot|s) - \pi(\cdot|s')}_1
     \leq  \abs{\cA} \log \abs{\cA}  L_Q \norm{s - s'} / \epsilon,
\end{align*}
where  $L_Q$ is defined as in Theorem \ref{theorem_smooth_vals}.
\end{theorem}

Notably, the proof of Theorem \ref{thrm_existence_smooth_policy} relies on the key observation that any smooth $Q$-function satisfying \eqref{eq_smooth_q} can be fed into the softmax operator, which induces a smooth policy.
This observation also provides a way for value-based methods (e.g., Q-learning) to learn a smooth policy.
Namely, one can first learn a smooth surrogate of the optimal $Q$-function, and then feed the learned surrogate into the softmax operator to induce a close-to-optimal policy that is also smooth.

{\bf $\bullet$ Robustness against observation noise with smooth policies.}
We have so far established that smooth policies naturally exist in a smooth environment as close-to-optimal policies, and thus smoothness serves as a strong prior for policy  search.
We will further demonstrate that the benefits of smooth policy go beyond boosting learning efficiency, by bringing in the additional advantage of robustness against observation uncertainty.

\begin{theorem}\label{theorem_v_p}
Let $\pi(a|s)$ be $L_{\pi}$-smooth policy.
For any perturbation sequence $\cbr{\delta^t_{s}}_{t \geq 0, s \in \cS}$,
define a perturbed policy (non-stationary) $\tilde{\pi} = \{\tilde{\pi}_t\}_{t \geq 0}$ by \begin{align*}
    \tilde{\pi}_t(a|s) = \pi(a|s + \delta^t_{s}),
\end{align*} 
with $\norm{\delta^t_{s}} \leq \epsilon$ for all $t \geq 0$.
Accordingly,  define the value function of the non-stationary policy $\tilde{\pi}$ 
\begin{align*}
    V^{\tilde{\pi}}(s)
     = & \EE \big[
    \tsum_{t=0}^\infty \gamma^t r(s_t, a_t) | s_0 = s, a_t \sim \tilde{\pi}_t(\cdot | s_t+\delta^t_{s_t}), \nonumber \\
    & ~~~~~~ s_{t+1} \sim \PP(\cdot |s_t, a_t)
    \big].
\end{align*}
Then
we have  
$
   \abs{ V^{\pi}(s) - V^{\tilde{\pi}}(s)} \leq \tfrac{  2 L_\pi \epsilon }{(1-\gamma)^2} ,
$
Similarly, we have 
\begin{align*}
  \abs{ Q^{\pi}(s,a) - Q^{\tilde{\pi}}(s,a)} \leq \tfrac{  2 L_\pi \epsilon }{(1-\gamma)^2} ,
\end{align*}
where $Q^{\tilde{\pi}}$ is defined similarly as $V^{\tilde{\pi}}$.
\end{theorem}

Theorem \ref{theorem_v_p} establishes the following fact: for a discounted MDP with finite state and finite action space,  the value of the policy when providing the perturbed state is close to the value of the policy when given the non-perturbed state, provided the policy is Lipschitz continuous in its state. 
As an important implication, the learned smooth policy will be robust in the state observation, in the sense that the accumulated reward will not deteriorate much when noisy, or even adversarially constructed state observations are given to the policy upon decision making.

We emphasize that Theorem \ref{theorem_v_p} holds without any smoothness assumption on the transition or the reward function.
It should also be noted that there are various notions of robustness in MDP, e.g., robustness against changes in transition kernel \citep{ruszczynski2010risk, li2022first},
which we defer as future investigations. 

Before we conclude this section, we briefly remark on certain generality of our discussion.

\begin{remark}[{\bf Applicability to MARL}]\label{remark_marl_extension}
The discussions in this section do not depend on the size of the state space, and apply to the multi-agent
setting without any change. To see this, note that our discussion holds for any discrete state-action space. Setting $\cS$ as the joint state space and $\cA$ as the joint action space, then the obtained results trivially carry over to the cooperative MARL setting.
\end{remark} 

\begin{remark}[{\bf Environments with approximate smoothness}]\label{remark_approx_smooth}
Many environments are partially smooth, in the sense that the transition or the reward is non-smooth only on a small fraction of the state space. 
Typical examples include the Box2D environment \citep{1606.01540},
 where the agent receives smooth reward when in non-terminal states (airborne for Lunar Lander), and receives a lump-sum reward in the terminal state (land/crash) -- a vanishing fraction of the entire state space.
 Given the environment being largely smooth, it should be expected that for most states the optimal policy is locally smooth. 
 Consequently, inducing a smoothness prior serves as a natural regularization to constrain the search space when solving these environments, without incurring a large bias.
\end{remark}

\begin{remark}[{\bf Non-smooth environments}]\label{remark_non_smooth}
From the perspective of robust statistics, achieving robustness often necessitates a certain level of smoothness in the learned policy, regardless of the smoothness of the optimal policy. In scenarios where the environment itself is non-smooth, the optimal policy can also be non-smooth. However, it is important to note that such non-smooth optimal policies are typically not robust. This means that by trading-off between the approximation bias and robustness, the smooth policy learnt by out method has the potential to outperform non-smooth policies in perturbed environments.
\end{remark}

\section{Wide Networks}
\label{app:wide}

Section \ref{sec:theory} tells us that smooth and close to optimal policies exist under certain conditions and ERNIE provides algorithms to find them. Now, a practical question remains: can neural networks be used to learn such policies? We show that as long as the width is sufficiently large, there exists a neural network with the desired optimality and smoothness properties. This finding further supports ERNIE's deployment as a tool for practical deep MARL.

Before we continue with further analysis, we will first introduce some necessary preliminaries. Specifically, we consider the Sobolev space, which contains a class of smooth functions \citep{brezis2011functional}.
\begin{definition}
Given $\alpha\geq 0$ and domain $\Omega\subset\RR^d$, we define the Sobolev space $W^{\alpha, \infty}(\Omega)$ as
\begin{align*}
W^{\alpha,\infty}(\Omega)=\big\{ &f\in L^\infty(\Omega): D^{\bm{\alpha}}f\in L^\infty(\Omega),~\forall~ |\bm{\alpha}|\leq \alpha\big\},
\end{align*}
where $D^{\bm{\alpha}} f = \frac{\partial^{|\balpha|} f}{\partial x_{1}^{\alpha_1} \cdots \partial x_D^{\alpha_D}}$ with multi-index $\bm{\alpha} = [\alpha_1, ..., \alpha_D]^\top \in \NN^D$.

For $f\in W^{\alpha,\infty}(\Omega)$, we define its Sobolev norm as
\begin{align*}
\textstyle \|f\|_{W^{\alpha,\infty}(\Omega)}=\max_{|\bm{\alpha}|\leq \alpha} \|D^{\bm{\alpha}}f\|_{L^{\infty}(\Omega)}
\end{align*}
\end{definition}
The Sobolev space has been widely investigated in the existing literature on function approximation of neural networks. For a special case $\alpha=1$, $\|f\|_{W^1,\infty}<\infty$ implies both the function value and its gradient are bounded.

We consider an $L$-layer ReLU neural network
\begin{align}\label{eq:reluf}
f(x) = W_L \cdot \sigma(\cdots \sigma(W_1 s + b_1) \cdots) + b_L,
\end{align}
where $W_1, \dots, W_L$ and $b_1, \dots, b_L$ are weight matrices and intercept vectors of proper sizes, respectively, and $\sigma(\cdot)=\max\{\cdot,0\}$ denotes the entry-wise ReLU activation. We denote $\cF$ as a class of neural networks:
\begin{align}\label{eq:classf}
\cF(L,p)  = &\big\{f ~|~ f(x) \textrm{ in the form \eqref{eq:reluf} with $L$-layers}\notag\\
&\textrm{and width bounded by $p$}\}.
\end{align}

We next present the function approximation results.
\begin{theorem}[Function approximation with Lipschitz continuity]\label{approx-thm-text}
Suppose that the target function $f^*$ satisfies
\begin{align*}
f^*\in W^{\alpha,\infty}\left(\Omega\right) \quad \text{and} \quad \|f^*\|_{W^{\alpha,\infty}\left(\Omega\right)}\leq 1
\end{align*}
for some $\alpha\geq 2$. Given a pre-specified approximation error $\epsilon$, there exists a neural network $\tilde{f}\in\cF(L,p)$ 
with $L = \tilde{O}(\log(\epsilon^{-1}))$ and $p= \tilde{O}(\epsilon^{-\frac{d}{\alpha-1}})$, such that
\begin{align*}
\|\tilde{f} - f^* \|_\infty \leq \epsilon\quad\textrm{and}\quad \|\tilde{f}\|_\mathrm{Lip} \leq 1+\sqrt{d}\epsilon^{1-1/\alpha},
\end{align*}
where $\tilde{O}$ hides some negligible constants or log factors and $\|\tilde{f}\|_\mathrm{Lip}$ denotes the Lipschitz constant of $\tilde{f}$.
\end{theorem}
For reinforcement learning, $f^*$ in Theorem \ref{approx-thm-text} can be viewed as either the near-optimal smooth policy $\pi^*$ or optimal smooth action-value function $Q^*$, and the input can be viewed as the state $s$ or the state-action pair $(s,a)$. As can be seen, a wider neural network not only better approximates a smooth target function $f^*$ well, but also further reduces the upper bound of its Lipschitz constant, which leads to a more robust policy. Moreover, we can certify the existence of a neural network $\tilde{f}$ such that $\|\tilde{f}\|_\mathrm{Lip}$ is below $2$, given a sufficient width $p=\tilde{O}\left(d^{\frac{d\alpha}{2(\alpha-1)^2}}\right)$. This result indicates that when training policies with the ERNIE algorithm, we should use wide neural networks.

%% file: source_files/method.tex
Section \ref{sec:theory} shows that promoting smoothness leads to both robust and high-performing policies for smooth environments, which can be achieved by sufficiently wide neural networks. 
Based on this insight, we propose
our robust MARL framework, adv\textbf{E}rsarially \textbf{R}egularized multiage\textbf{N}t re\textbf{I}nforcement l\textbf{E}arning (ERNIE).

\subsection{Learning Robust Policy with ERNIE}
\label{sec:learning-robust}
Section \ref{sec:theory} shows that the robustness of a policy depends on its Lipschitz constant. Therefore, in ERNIE we propose to control the Lipschitz constant of each policy with adversarial regularization. 

Given a policy $\pi_{\theta_k}$, where $k$ is the agent index, the ERNIE regularizer is defined by
\begin{align}
\label{inner-max}
    R_\pi( o_k; \theta_k) = \mathop{\text{max}}_{||\delta|| \leq \epsilon} D(\pi_{\theta_k}(o_k + \delta), \pi_{\theta_k}(o_k)).
\end{align}
Here $\delta(o_k, \theta_k)$ is a perturbation adversarially chosen to maximize the difference between the policy's output for the perturbed observation $o_k+\delta(o_k, \theta_k)$ and the original observation $o_k$. In this case $\epsilon$ controls the perturbation strength and $||\cdot||$ is usually taken to be the $\ell_2$ or $\ell_\infty$ norm. Note that $R_{\pi}( o_k; \theta_k)$ essentially measures the local Lipschitz smoothness of policy function $\pi_\theta$ around the observation $o_k$, defined in metric $D(\cdot, \cdot)$. Therefore minimizing $R_{\pi}( o_k; \theta_k)$ will encourage the policy to be smooth.

 Regularization Eq.~\eqref{inner-max} allows straightforward incorporation into  MARL algorithms that directly perform policy search.
For actor-critic based policy gradient methods, the regularizer Eq.~\eqref{inner-max} can be directly included into the objective for updating the actor (policy) networks.
When optimizing stochastic policies (e.g., MAPPO \citep{chao2021surprising}),
 $D$ can be taken to be the KL divergence and 
for deterministic policies (e.g., MADDPG \citep{lowe2017multi} or Q-learning \citep{tsitsiklis1996analysis}), 
 we set $D$ to be the $\ell_p$ norm.

More concretely, let $\cL_{\mathrm{}}(\theta)$ denote the policy optimization objective, i.e., the negative weighted value function of the policy.
We then augment $\cL_{\mathrm{}}(\theta)$ with Eq.~\eqref{inner-max}, and minimize the regularized objective
\begin{equation}
\label{eq:reg_coma}
     \min_\theta \cF(\theta) =  \cL_{\mathrm{}}(\theta) + \lambda \tsum_{n=1}^N\EE_{\pi_n}\big[R_{\pi}( o_n; \theta_n)\big],
\end{equation}
where $\lambda$ is a hyperparameter. We remark that \citet{shen2020deep} has explored similar regularization for single-agent RL (with a goal of improving sample efficiency), but as we explain in sections \ref{sec:stackelberg}, \ref{sec:action}, and \ref{sec:mean-field}, successful application to MARL robustness is highly non-trivial.

\subsection{Stackelberg Training with Differentiable Adversary}
\label{sec:stackelberg}
Although accurately solving Eq.~\eqref{eq:reg_coma} will result in a high-performing and robust policy, we note that Eq.~\eqref{eq:reg_coma} is a nonconvex-nonconcave minimax problem. In practice, we can use multiple steps of projected gradient ascent to approximate the worse-case state perturbation $\delta(o_k, \theta_k)$, followed by one-step gradient descent for updating the policies/Q-function. Even though this optimization method already significantly improves robustness over the baseline algorithms, we observe that the training process could be quite unstable. We hypothesize that the intrinsic instability of MARL algorithms due to simultaneous updates of multiple agents is greatly amplified by the non-smooth landscape of adversarial regularization. 

To promote a more stable straining process, we propose to reformulate adversarial training in ERNIE as a Stackelberg game. The reformulation defines adversarial regularization as a leader-follower game \citep{von2010market}:
\begin{align}
    R_\pi( o, \delta^K_\theta(o); \theta) & =   D\big(\pi_\theta(o + \delta^K(o, \theta)), \pi_\theta(o)\big) \label{eq:stack} \\
    \mathrm{s.t.}~~~ & \delta^K(o, \theta) = \underbrace{U_\theta \circ U_\theta \circ \cdots \circ U_\theta}_{\text{K-fold composition}}(\delta^0(o, \theta)). \nonumber
\end{align}
Here $\circ$  denotes the operator composition (i.e $f \circ g = f(g(\cdot))$), and
\begin{align*}
\delta^{k+1}(o, \theta) = U_\theta (\delta^k(o,\theta)) = \delta^k(o,\theta) + \eta \nabla_\delta
D\left(\pi_\theta\left(o + \delta^k(o, \theta)\right), \pi_\theta(o)\right)
\end{align*}
is a one-step gradient ascent for maximizing the divergence of the perturbed and original observation.

Compared to the vanilla adversarial regularizer in Eq.~\eqref{inner-max}, the perturbation $\delta$ is treated as a function of the model parameter $\theta$.
This formulation allows the leader ($\theta$) to anticipate the action of the follower ($\delta$), since the follower's response given observation $o$ is fully specified by $\delta^K(o,\theta)$.
This structural anticipation effectively produces an easier and smoother optimization problem for the leader ($\theta$), whose gradient, termed the Stackelberg gradient, can be readily computed by 
\begin{align*}
\frac{\partial R_\pi(o, \delta^K_\theta(o); \theta)}{\partial \theta} =  \underbrace{\frac{\partial R_\pi(o, \delta^K, \theta)}{\partial\theta}}_{\text{\textcolor{red}{leader}}}
\hspace{0.25in}+ \underbrace{\frac{\partial R_\pi(o, \delta^K(\theta), \theta)}{\partial\delta^K(\theta)} \frac{\delta^K(\theta)}{\partial \theta}}_{\text{\textcolor{blue}{leader-follower interaction}}}
\end{align*}
Note that the gradient used in Eq.~\eqref{eq:reg_coma} only contains the ``leader'' term, such that interaction between the model $\theta$ and the perturbation $\delta$ is ignored.
The computation of the Stackelberg gradient can be reduced to Hessian vector multiplication using finite difference method \citep{pearlmutter2008reverse}, which only requires two backpropogations and extra $\cO(d)$ complexity operation.
Thus no significant computational overhead is introduced for solving Eq.~\eqref{eq:stack}.

The benefit of Stackelberg training for MARL is twofold. First, a smoother optimization problem results in a more stable training process. This extra stability is essential given the inherent instability of MARL training. Second, giving the policy $\theta$ priority over the attack $\delta$ during the training process allows for a better training data fit than normal adversarial training allows. This better fit allows the MARL policies trained with Stackelberg training to perform better in lightly perturbed environments than those trained with normal adversarial regularization.

\subsection{Robustness against Malicious Actions}
\label{sec:action}
Given the complex interactions of agents within of a multi-agent system, a robust policy for any given agent should meet the criterion that the action made is not overly dependent on any small subset of agents.
This is particularly the case when the agents are homogeneous in nature
\citep{wang2020breaking, li2021permutation}, and thus there should be no notion of {\it coreset agents} 
in the decision-making process that could heavily influence the actions of other agents.  
 We proceed to show how ERNIE could be adopted to induce such a notion of robustness.

The core idea of ERNIE for this scenario is to 
encourage policy/Q-function smoothness with respect to {\it joint actions}. 
Similar to our treatment in Eq.~\eqref{inner-max}, we now seek to promote learning a Q-function that yields a consistent value when perturbing the actions for any small subset of agents. Specifically, for discrete action space, we define a regularizer on the global Q-function as
\begin{align}
    \label{eq:action-reg}
    R^A_\omega( s, \mathbf{a}) = \mathop{\text{max}}_{D(\mathbf{a}, \mathbf{a'}) \leq K} ||Q(s, \mathbf{a}; \omega) - Q(s, \mathbf{a'}; \omega)||_2^2,
\end{align}
where $D(\mathbf{a}, \mathbf{a'}) = \tsum_{i} I(\mathbf{a}_i \neq \mathbf{a'}_i)$. 
The regularizer Eq.~\eqref{eq:action-reg} seeks to compute the worst subset of changed actions with cardinality less than $K$. 
For continuous action spaces, one could replace the metric $D$ in Eq.~\eqref{eq:action-reg} by a differentiable metric defined over the action space (e.g., $\norm{\cdot}_2$-norm), and then evaluate the regularizer with projected gradient ascent.

To evaluate the adversarial regularizer for the discrete action space, we propose to solve Eq.~\eqref{eq:action-reg} in a greedy manner by finding the worst-case change in the action of a single agent at a time, until the action of $K$ agents is changed. 
Specifically, at each training step, we search through all the agents/actions and then pick the actions that produce the top-$K$ changes in the Q-function, resulting in a $\mathcal{O}(|\mathcal{A}|*N*K)$ computation.
Our complete algorithm can be found in Appendix \ref{app:action}, and we find that in our numerical study, perturbing the action of a single agent ($K = 1$) is sufficient for increased robustness.

Similar to the regularizer in Eq.~\eqref{inner-max}, the regularizer in Eq.~\eqref{eq:action-reg} provides the benefits of Lipschitz smoothness (with respect to the Hamming distance) and data augmentation with adversarial examples. If the behavior of a few agents changes (either maliciously or randomly), the behavior of policies trained by conventional methods may change drastically. On the other hand, policies trained by our method will continue to make reasonable decisions, resulting in more stable performance (see section \ref{sec:experiments}).


\subsection{Extension to Mean-field MARL}
\label{sec:mean-field}
\input{source_files/mean-field.tex}

%% file: source_files/mean-field.tex
MARL algorithms have been known to suffer from the curse of many agents \citep{wang2020breaking}, as the search space of policies and value functions grows exponentially w.r.t. the number of agents.
A practical approach to tackle this challenge of scale is to adopt the mean-field approximation, which views each agent as realizations from a distribution of agents. 
This distributional perspective requires a distinct treatment of ERNIE applied to the mean-field setting.

Mean-field MARL avoids the curse of many agents by approximating the interaction between each agent and the global population of agents with that of an agent and the average agent from the population. 
In particular, we can approximate the action-value function of agent $j$ as
$
Q^j(\mathbf{s}, \mathbf{a}) = Q^j(s_j, d_s, a_j, \bar a_j),
$
where \textbf{a} is the global joint action, \textbf{s} is the global state, $\bar a_j$ is the average action of agent $j$'s neighbors, and $d_s$ is the empirical distribution of states over the population. 
Such an approximation has found widespread applications in practical MARL algorithms \citep{yang2018mean, li2019efficient, li2021permutation, gu2021mean}, and can be motivated in a principled fashion for agents of homogeneous nature \citep{wang2020breaking, li2021permutation}.

To learn robust and scalable policies, we extend ERNIE to the mean-field setting by applying adversarial regularization to the approximation terms $d_s$ and $\bar a^j$. It is important to note that as the terms $d_s$ and $d'_s$ represent distributions over states, we bound the attack by the Wasserstein distance \citep{ruschendorf1985wasserstein}. In what follows we only apply the regularizer to $d_s$ for simplicity. This leads to a new regularizer defined over the mean field state
\begin{align*}
    R^Q_\cW( s; \theta) =
    \mathop{\text{max}}_{\cW(d'_s, d_s)\leq\epsilon} \sum_{a\in \cA} \norm{Q_\theta(s, d'_s, a) - Q_\theta(s, d_s, a)}_2^2,
\end{align*}
where $\cW$ denotes the Wasserstein distance metric. Since the explicit Wasserstein constraint may be difficult to optimize in practice, we can instead enforce the constraint through regularization, as displayed in Appendix \ref{app:wass}.

%% file: source_files/experiments.tex
We conduct extensive experiments to demonstrate the effectiveness of our proposed framework. In each environment, we evaluate MARL algorithms trained by the ERNIE framework against baseline robust MARL algorithms. To evaluate robustness, we train MARL policies in a non-perturbed environment and evaluate these policies in a perturbed environment. The reported results are gathered over five runs for each algorithm. Given the space limit, we put additional results in Appendix \ref{app:additional}.

\textbf{Traffic Light Control.} In this scenario, cooperative agents learn to minimize the total travel time of cars in a traffic network. We use QCOMBO \citep{zhang2019integrating} and COMA \citep{foerster2018counterfactual} (results in appendix \ref{app:additional}) as our baseline algorithms and conduct experiments using the Flow framework \citep{wu2017flow}. 
We train the MARL policies on a two-by-two grid (four agents). We then evaluate the policies on a variety of realistic environment changes, including different car speeds, traffic flows, network topologies, and observation noise. In each setting, we plot the reward for policies trained with ERNIE, the baseline algorithm (QCOMBO), and another baseline where the attack $\delta$ is generated by a Gaussian random variable (Baseline-Gaussian, see Appendix \ref{app:gaussian}). Implementation details can be found in Appendix \ref{app:traffic}.

Figure \ref{fig:QCOMBO}, \ref{fig:large}, and \ref{fig:topology} show that the baseline algorithm is vulnerable to small changes in the training environment (higher reward is better). On the other hand, ERNIE achieves more stable reward on each environment change. 
This observation confirms that the ERNIE framework can improve robustness against observation noise and changing transition dynamics. The Gaussian baseline performs well on some environment changes, like when the observations are perturbed by Gaussian noise, but performs poorly on other environment changes, like when the car speed is changed. We hypothesize that while some environment changes may be covered by Gaussian perturbations, other environment changes are unlike Gaussian perturbations, resulting in a poor performance from this baseline.

\textbf{Robustness Against Malicious Actions.} We also evaluate the extension of ERNIE to robustness against changing agent behavior, which we refer to as ERNIE-A. To change agent behavior, we adversarially change the action of a randomly selected agent a small percentage of the time, i.e. 5\% or 3\% of the time. As can be seen in Figures \ref{fig:malicious3} and \ref{fig:malicious5}, the two baseline algorithms perform poorly when some agent's behavior changes. In contrast, ERNIE-A is able to maintain a higher reward.

\begin{figure*}[htb!]
\centering
     \begin{subfigure}{0.24\textwidth}
         \centering
         \includegraphics[width=\textwidth]{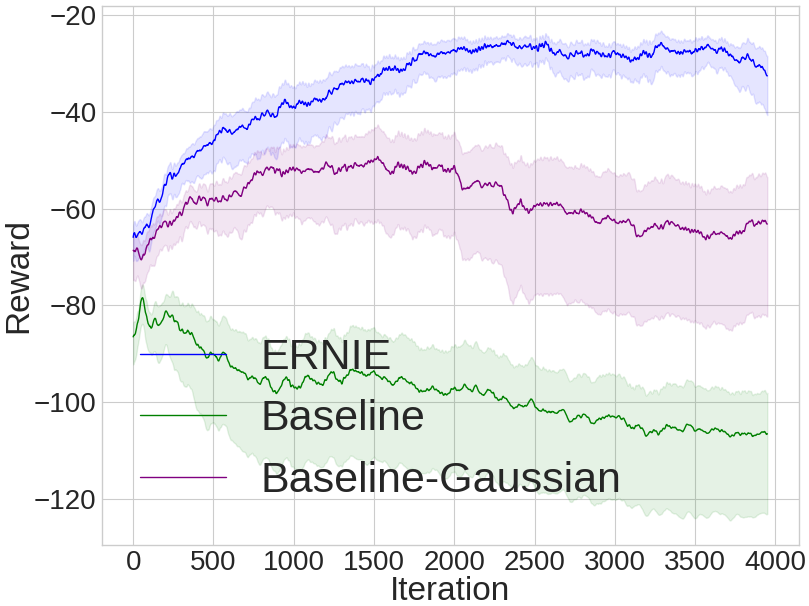}
         \caption{Gaussian Noise (0.01)}
     \end{subfigure}
     \hspace{0.05\textwidth}
     \begin{subfigure}{0.235\textwidth}
         \centering
         \includegraphics[width=\textwidth]{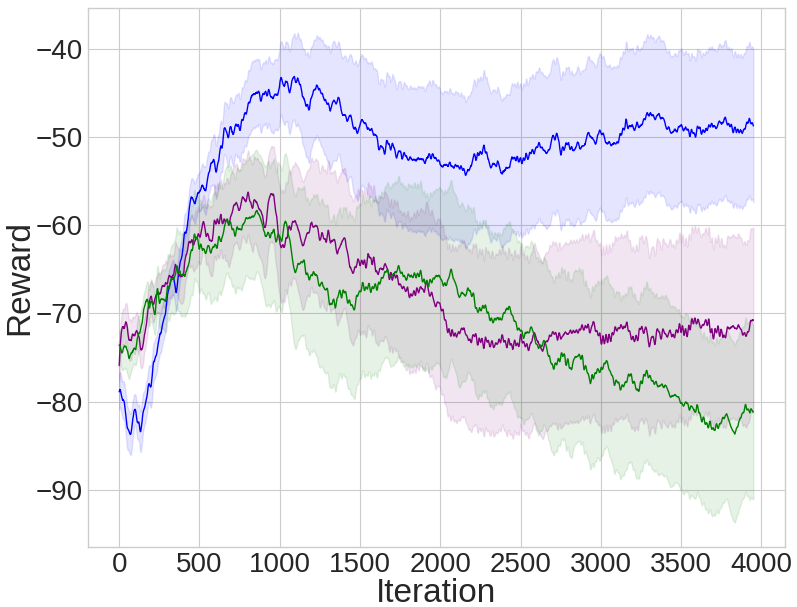}
         \caption{Different Speed}
     \end{subfigure}
     \hspace{0.05\textwidth}
    \begin{subfigure}{0.243\textwidth}
         \centering         \includegraphics[width=\textwidth]{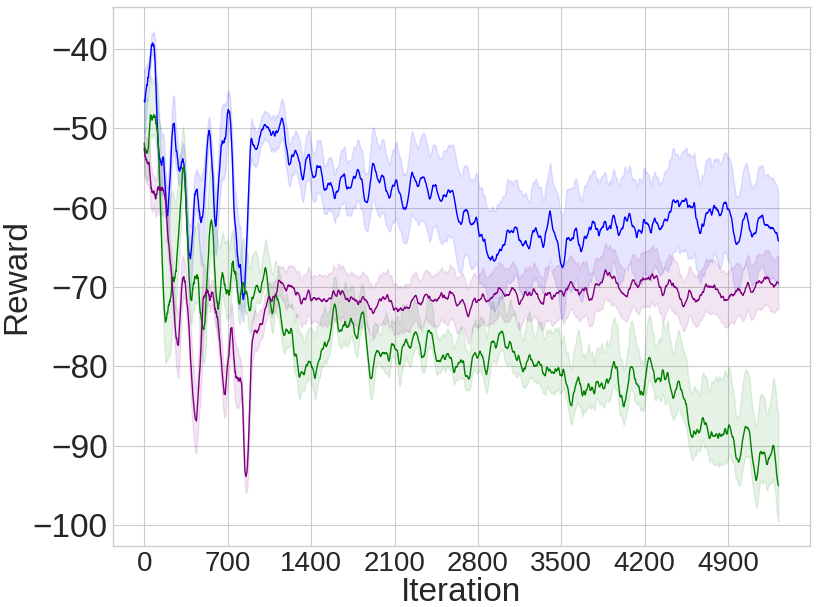}
         \caption{Traffic Flow (Flow-3)}
     \end{subfigure}
    \caption{Evaluation curves on different environment changes for traffic light control.}
    \label{fig:QCOMBO}
\end{figure*}

\begin{figure*}[htb!]
\centering
     \begin{subfigure}{0.22\textwidth}
         \centering
         \includegraphics[width=\textwidth]{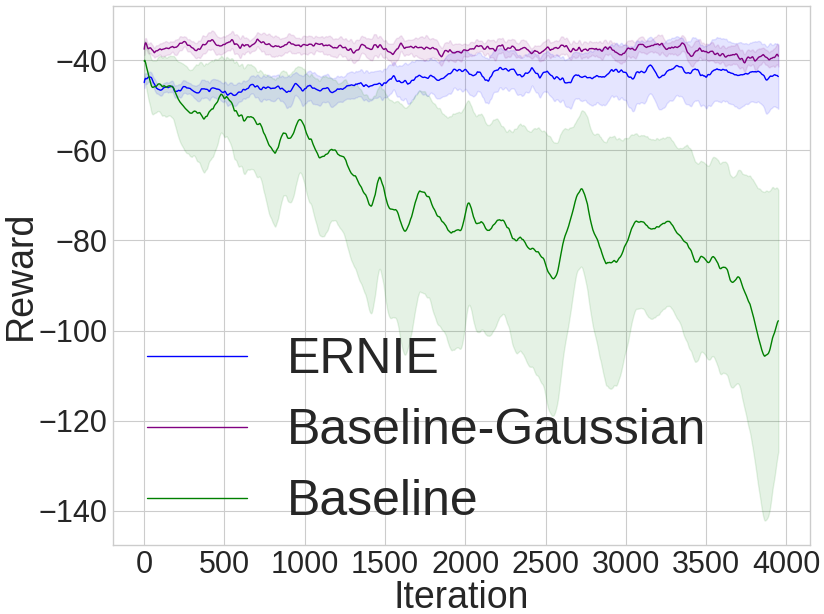}
         \caption{Larger Network}
         \label{fig:large}
     \end{subfigure}
     \hspace{0.001\textwidth}
     \begin{subfigure}{0.22\textwidth}
         \centering
         \includegraphics[width=\textwidth]{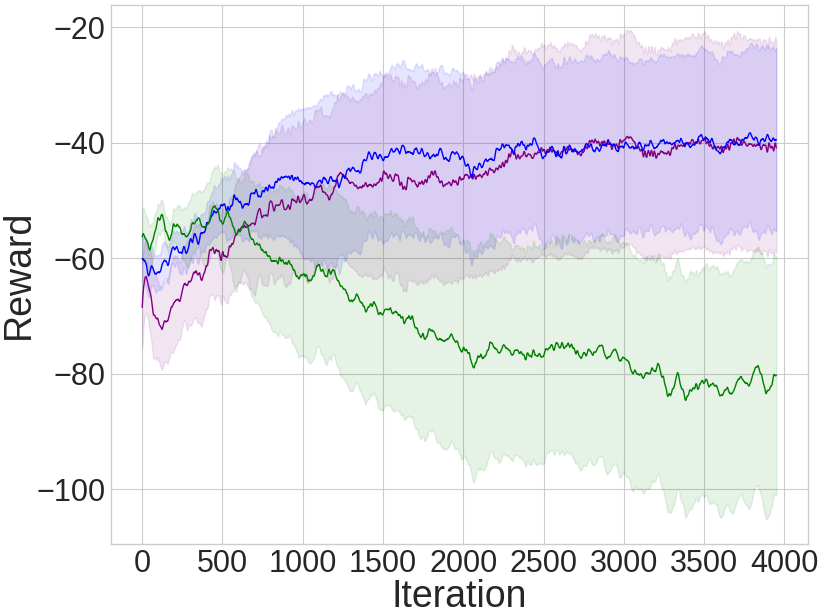}
         \caption{Different Topology}
         \label{fig:topology}
     \end{subfigure}
     \hspace{0.001\textwidth}
    \begin{subfigure}{0.22\textwidth}
         \centering         \includegraphics[width=\textwidth]{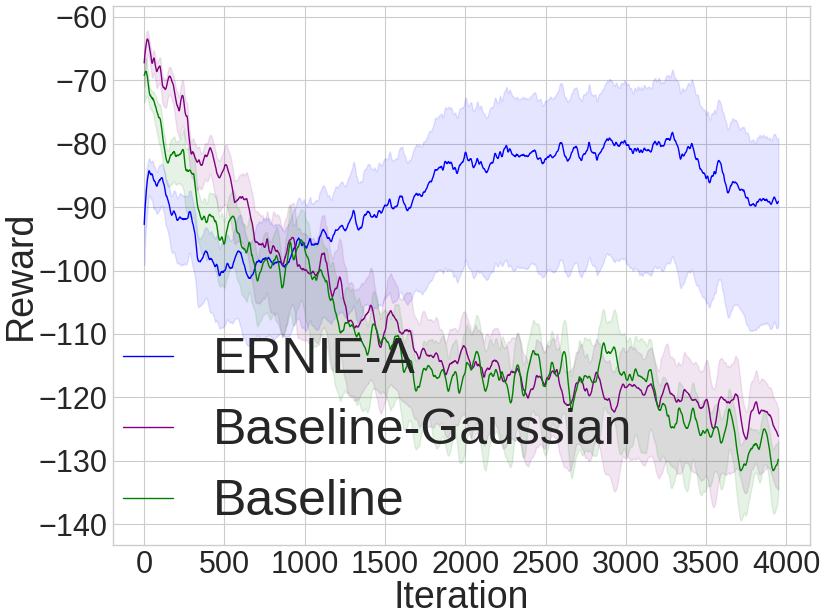}
         \caption{Actions (3\%)}
         \label{fig:malicious3}
     \end{subfigure}
     \hspace{0.001\textwidth}
    \begin{subfigure}{0.22\textwidth}
         \centering         \includegraphics[width=\textwidth]{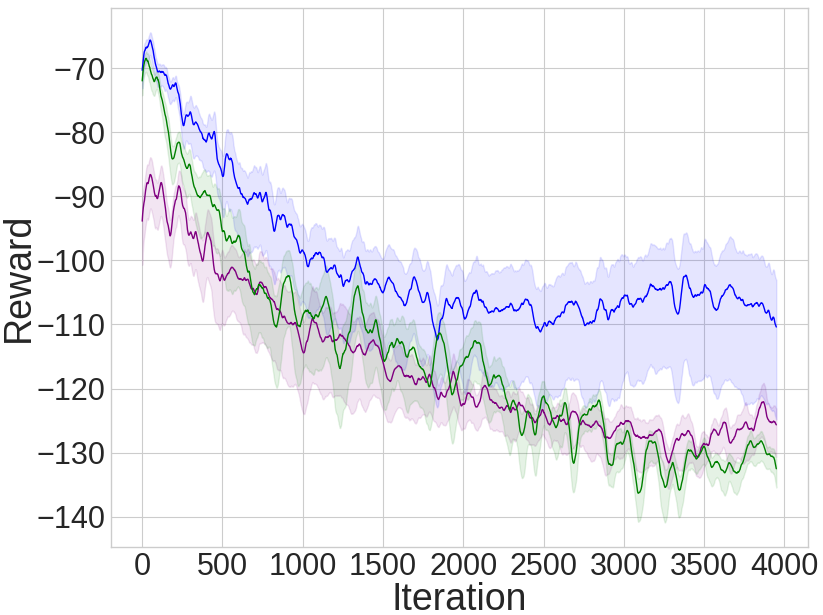}
         \caption{Actions (5\%)}
         \label{fig:malicious5}
     \end{subfigure}
    \caption{Performance on changed traffic network topologies and with malicious agents. In Figures (c) and (d) we perturb the actions according to the specified percentages.}
\end{figure*}

\textbf{Particle Environments.} We evaluate ERNIE on the cooperative navigation, predator-prey, tag, and cooperative communication tasks. In each setting, we investigate the performance of the baseline algorithm (MADDPG), ERNIE, M3DDPG, and the baseline-gaussian in environments with varying levels of observation noise. We also compare ERNIE to the RMA3C algorithm proposed in  \citep{han2022solution}. In Figure \ref{fig:particle}, we find ERNIE performs better or equivalently than MADDPG in all settings. Surprisingly, M3DDPG can provide some robustness against observation noise, even though it is designed to provide robustness against malicious actions.

\begin{figure*}[htb!]
\centering
     \begin{subfigure}{0.22\textwidth}
         \centering
         \includegraphics[width=\textwidth]{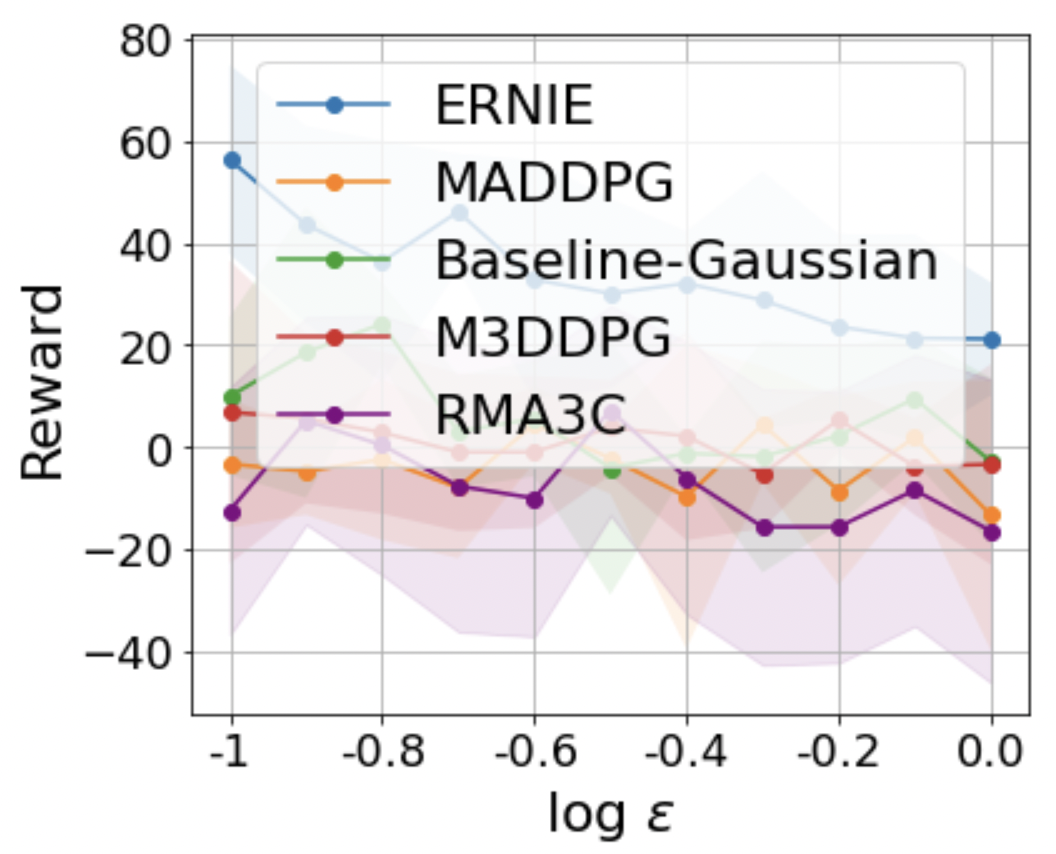}
         \caption{Covert Comm.}
     \end{subfigure}
     \hspace{0.001\textwidth}
     \begin{subfigure}{0.22\textwidth}
         \centering
         \includegraphics[width=\textwidth]{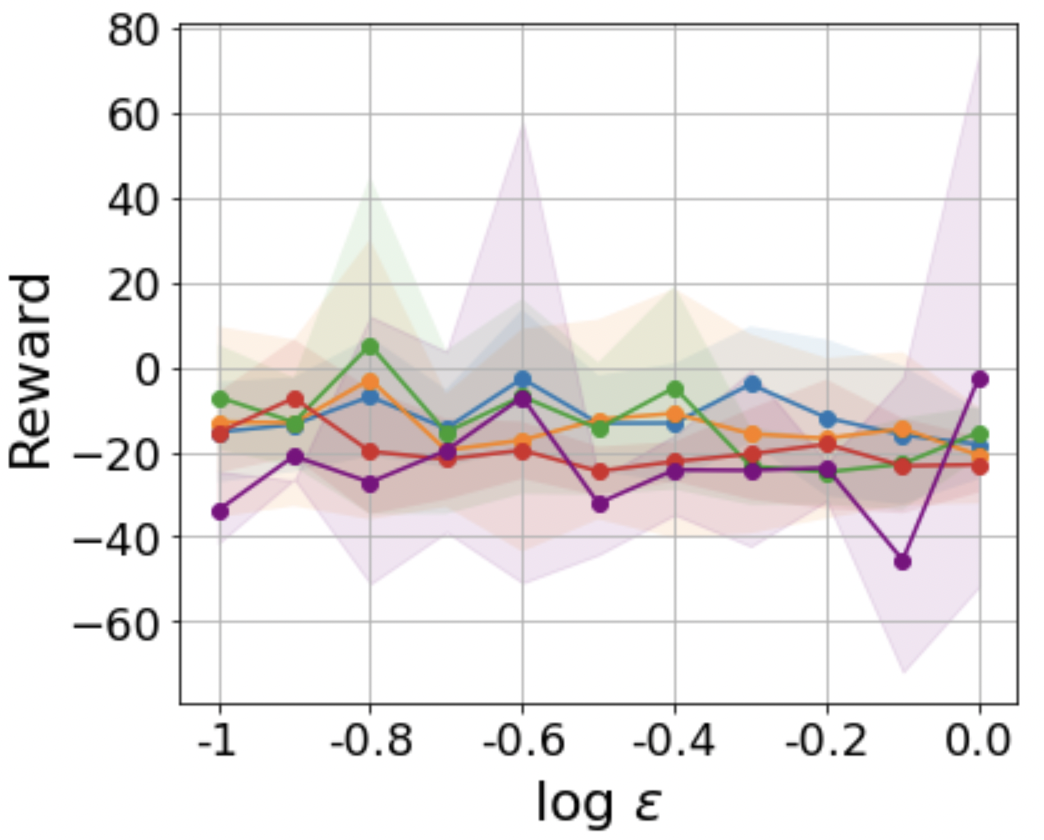}
           \caption{Tag}
     \end{subfigure}
     \hspace{0.001\textwidth}
    \begin{subfigure}{0.23\textwidth}
         \centering         \includegraphics[width=\textwidth]{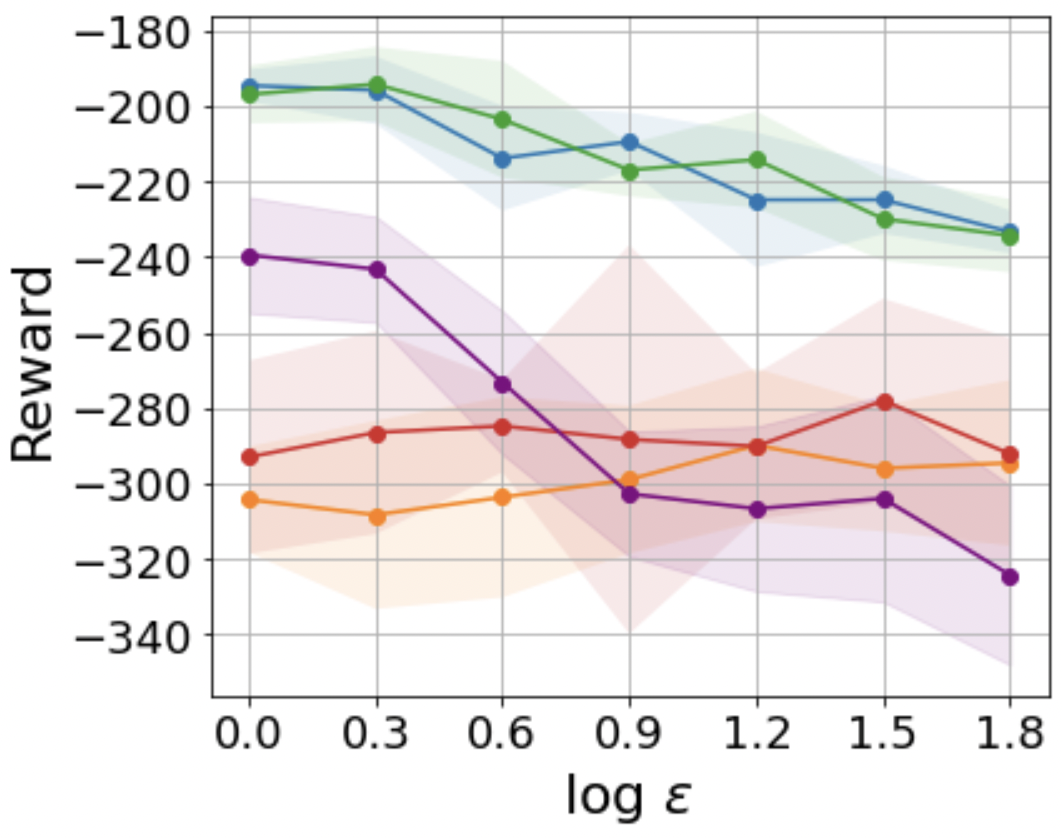}
           \caption{Navigation}
     \end{subfigure}
     \hspace{0.001\textwidth}
    \begin{subfigure}{0.22\textwidth}
         \centering         \includegraphics[width=\textwidth]{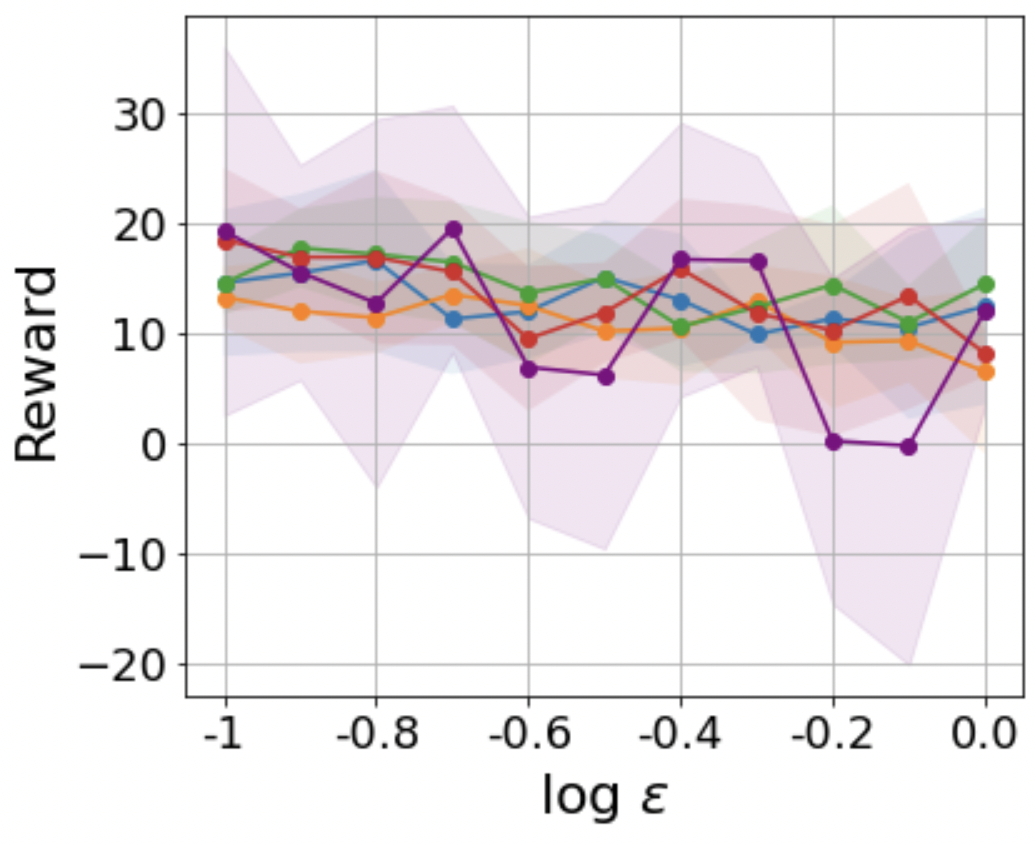}
           \caption{Predator Prey}
     \end{subfigure}
    \caption{Training reward versus noise level ($\epsilon$) in the evaluation environment for the particle games.}
    \label{fig:particle}
\end{figure*}

\textbf{Mean-field MARL.}
We evaluate the performance of the mean-field ERNIE extension on the cooperative navigation task \citep{lowe2017multi} with different numbers of agents. We compare the performance of the baseline algorithm, ERNIE, and M3DDPG under various levels of observation noise. We use mean-field MADDPG as our baseline and follow the implementation of \citep{li2021permutation}. 
As can be seen in Figure \ref{fig:mean-field}, ERNIE displays a higher reward and a slower decrease in performance across noise levels.

\begin{figure*}[htb!]
    \centering
     \begin{subfigure}[b]{0.24\textwidth}
         \centering
         \includegraphics[width=\textwidth]{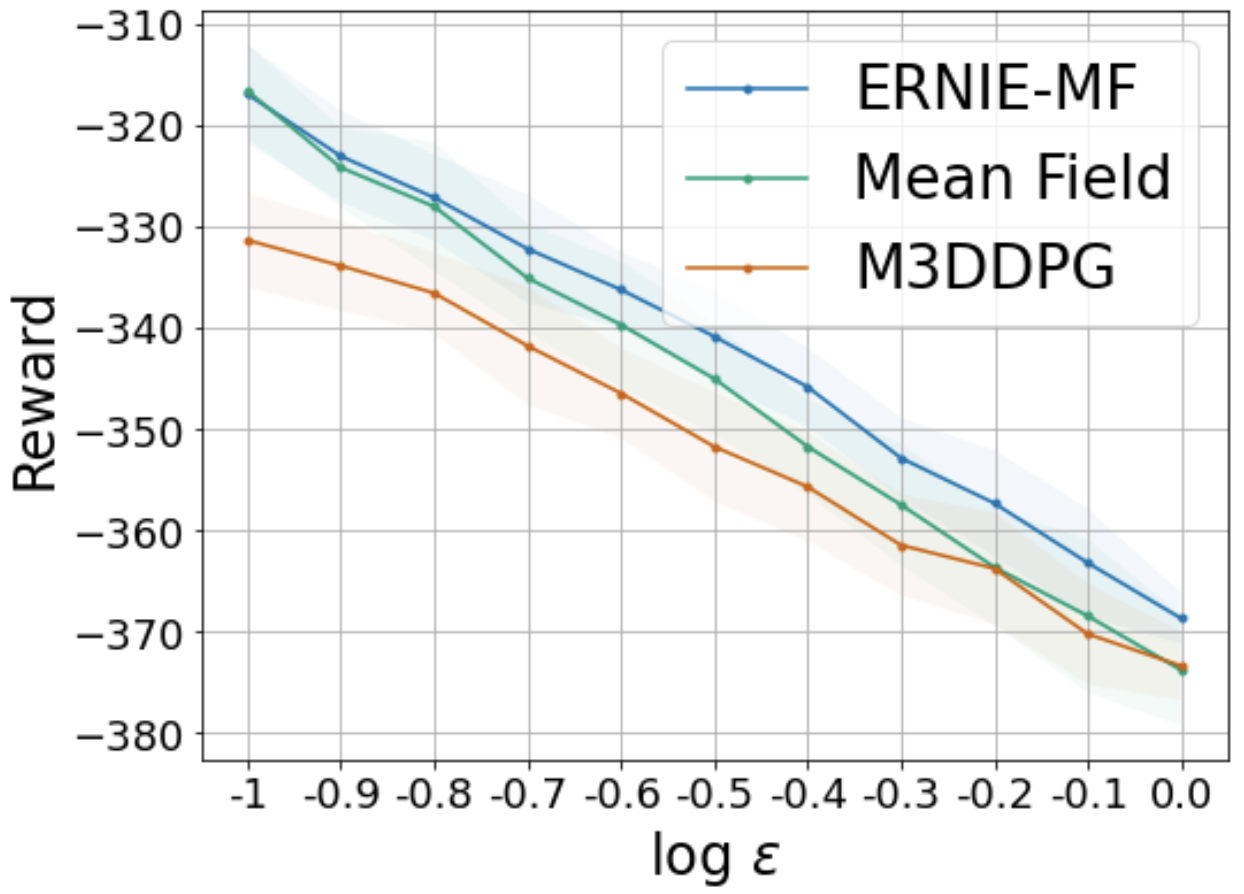}
         \caption{N=3}
     \end{subfigure}
     \begin{subfigure}[b]{0.24\textwidth}
         \centering
         \includegraphics[width=\textwidth]{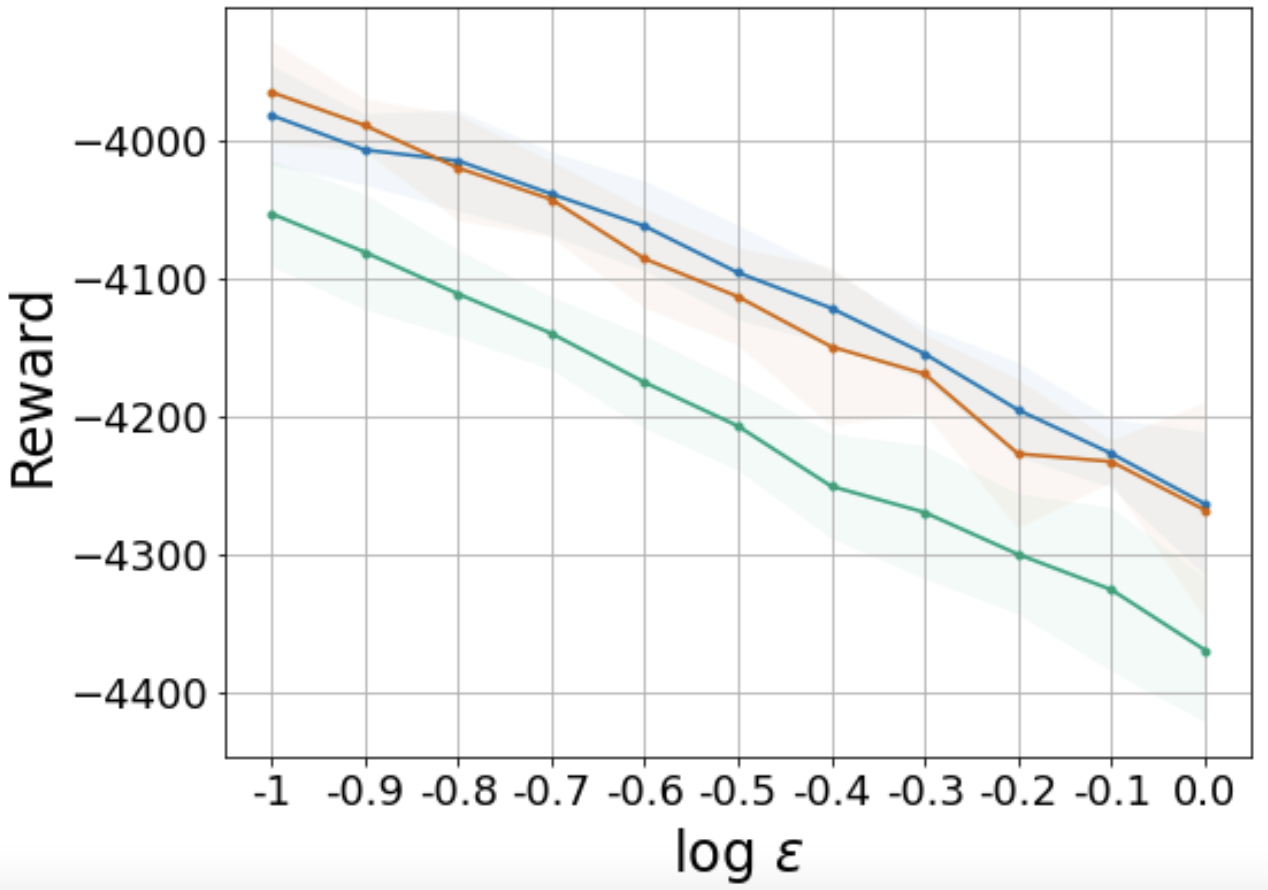}
         \caption{N=6}
     \end{subfigure}
     \begin{subfigure}[b]{0.24\textwidth}
         \centering
         \includegraphics[width=\textwidth]{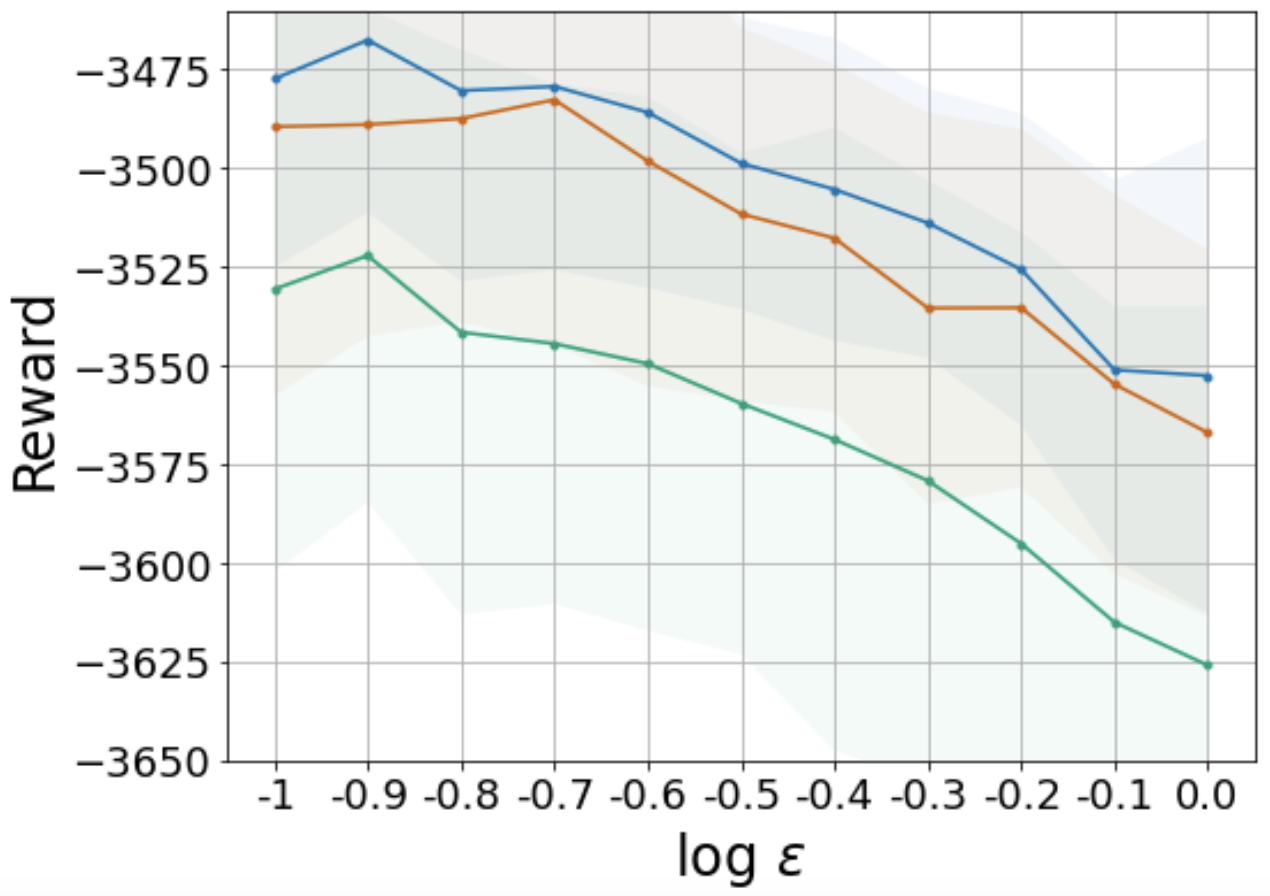}
         \caption{N=15}
     \end{subfigure}
     \begin{subfigure}[b]{0.22\textwidth}
     \label{fig:width}
         \centering
         \includegraphics[width=\textwidth]{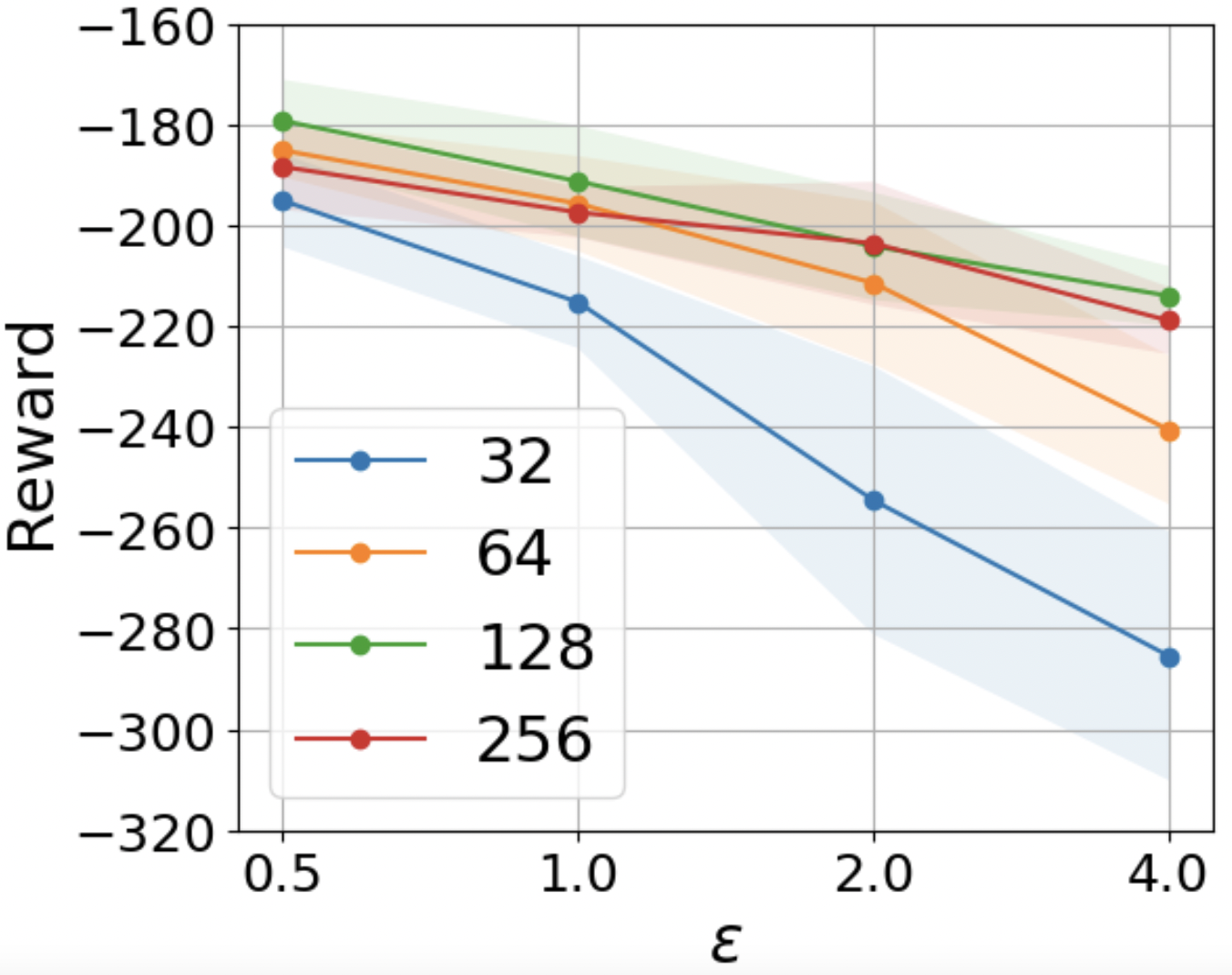}
         \caption{Model Width}
         \label{fig:width-net}
     \end{subfigure}
\caption{(a)-(c) Training reward versus noise level (mean $\pm$ standard deviation over 5 runs) with a various number of agents (N) (d) Network width and robustness.}
\label{fig:mean-field}
\end{figure*}
\begin{figure*}[htb!]
\centering
     \begin{subfigure}{0.21\textwidth}
         \centering
         \includegraphics[width=\textwidth]{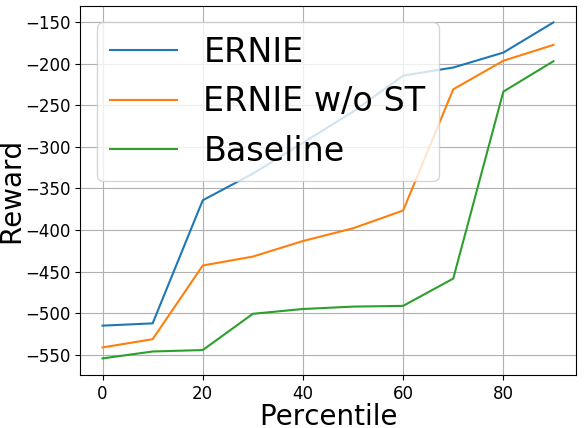}
         \caption{Different Speed}
        \label{fig:perc-speed}
     \end{subfigure}
     \hspace{0.001\textwidth}
     \begin{subfigure}{0.21\textwidth}
         \centering
         \includegraphics[width=\textwidth]{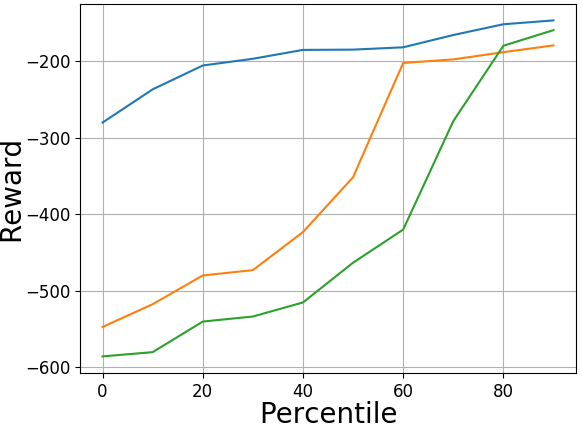}
           \caption{Observation Noise}
           \label{fig:perc-noise}
     \end{subfigure}
     \hspace{0.001\textwidth}
    \begin{subfigure}{0.22\textwidth}
         \centering         \includegraphics[width=\textwidth]{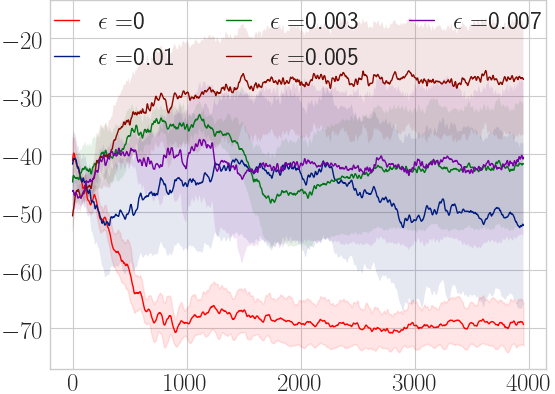}
           \caption{Sensitivity to $\epsilon$}
           \label{fig:sensitivity-epsilon}
     \end{subfigure}
     \hspace{0.001\textwidth}
    \begin{subfigure}{0.22\textwidth}
         \centering         \includegraphics[width=\textwidth]{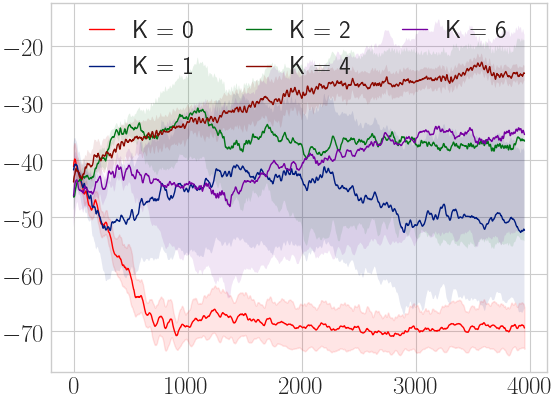}
           \caption{Sensitivity to $K$}
           \label{fig:sensitivity-k}
     \end{subfigure}
    \caption{Sensitivity and ablation experiments.}
\end{figure*}

\textbf{Hyperparameter Study.}
We investigate the sensitivity of ERNIE to the hyperparameters $K$ (the number of attack steps) and $\epsilon$ (the perturbation strength). We plot the performance of different hyperparameter settings in the traffic light control task with perturbed car speeds on three random seeds. From Figures \ref{fig:sensitivity-epsilon} and \ref{fig:sensitivity-k} we can see that adversarial training ($K>0$) outperforms the baseline $(K=0)$ for all $K$. Similarly, we can see that all values of $\epsilon$ outperform the baseline $(\epsilon=0)$. This indicates that ERNIE is robust to different hyperparameter settings of $\epsilon$ and $K$.

\textbf{Sensitivity and Ablation Study.}
The advantage of the ERNIE framework goes beyond improving the mean reward. To show this, we evaluate 10 different initializations of each algorithm in two traffic environments: one with different speeds and another environment with observation noise. We then sort the cumulative rewards of the learned policies and plot the percentiles in Figures \ref{fig:perc-speed} and \ref{fig:perc-noise}. Although the best-case performance is the same for all algorithms, ERNIE significantly improves the robustness to failure. As an ablation, we evaluate the effectiveness of ERNIE with and without the Stackelberg formulation of adversarial regularization (ST). ERNIE displays better performance in both settings, indicating that Stackelberg training can lead to a more stable training process. Ablation experiments in other environments can be found in Appendix \ref{app:able}.

\textbf{Robustness and Network Width.} In Section \ref{app:wide}, we show that in order to learn a robust policy with ERNIE, we should use a sufficiently wide neural network. Therefore, we evaluate the robust performance of ERNIE using  policy networks with 32, 64, 128, and 256 hidden units. We carefully tune their regularization parameters such that all networks perform similarly in the lightly perturbed environment. As seen in Figure \ref{fig:width-net}, when the perturbed testing environment deviates more from the training environment, the performance of the narrower policy networks (32/64 hidden units) significantly drops, while the wider networks (128/256 hidden units) are more stable.
We also observe that when the policy networks are sufficiently wide (128/256), their robustness is similar.

\textbf{Robotics Experiments.} Additional experiments in multi-agent drone control environments can be found in Appendix \ref{app:drone}, which further verify the enhanced robustness that ERNIE provides.